\documentclass[letterpaper,twocolumn,10pt]{article}
\usepackage{usenix}
\usepackage{enumitem}
\usepackage{multirow}
\usepackage{tikz}
\usepackage{amsmath,amssymb,amstext,amsthm} 
\usepackage{ifthen}
\DeclareSymbolFont{newfont}{OML}{cmm}{m}{it}
\DeclareMathSymbol{\varrho}{3}{newfont}{37}
\usepackage[
	n,
	operators,
	advantage,
	sets,
	adversary,
	landau,
	probability,
	notions,	
	logic,
	ff,
	mm,
	primitives,
	events,
	complexity,
	asymptotics,
	keys]{cryptocode}
\usepackage{adjustbox}
\usepackage{booktabs}
\usepackage{filecontents}
\newtheorem{theorem}{Theorem}
\newtheorem{definition}{Definition}
\newtheorem{lemma}{Lemma}

\usepackage[numbers,sort&compress]{natbib}
\usepackage{pifont}

\begin{document}

\date{}

\title{\Large \bf SoK: Analyzing Adversarial Examples: A Framework to Study Adversary Knowledge}

\author{
{\rm Lucas Fenaux}\\
University of Waterloo \\
lucas.fenaux@uwaterloo.ca
\and
{\rm Florian Kerschbaum}\\
University of Waterloo \\
florian.kerschbaum@uwaterloo.ca
} 
\maketitle

\begin{abstract}

Adversarial examples are malicious inputs to machine learning models that trigger a misclassification. This type of attack has been studied for close to a decade, and we find that there is a lack of study and formalization of adversary knowledge when mounting attacks. This has yielded a complex space of attack research with hard-to-compare threat models and attacks. We focus on the image classification domain and provide a theoretical framework to study adversary knowledge inspired by work in order theory. We present an adversarial example game, inspired by cryptographic games, to standardize attacks. We survey recent attacks in the image classification domain and classify their adversary's knowledge in our framework. From this systematization, we compile results that both confirm existing beliefs about adversary knowledge, such as the potency of information about the attacked model as well as allow us to derive new conclusions on the difficulty associated with the white-box and transferable threat models, for example, that transferable attacks might not be as difficult as previously thought. 

\end{abstract}

\section{Introduction}
Many machine learning (ML) models are currently deployed in critical environments like healthcare \cite{ml_healthcare}, self-driving cars \cite{GUPTA2021100057}, aerospace and aviation \cite{Ni_2022}, and many more. 
Model failures in these settings can directly incur grave consequences on human life.  While there exists a large body of literature on the safe and secure deployment of these models, 
the proposed solutions remain mainly academic, difficult to implement and replicate, and fall short of the security and safety standards we would typically expect from critical systems.
Despite all of this, the deployment of ML models continues to grow. This surge in deployment, while promising, introduces a pressing concern: the vulnerability of these models to security flaws, both known and yet to be uncovered. 
The potential consequences of such vulnerabilities extend far beyond the limits of academia, with economic and societal implications.

Adversarial example attacks are one such attack, originally discovered by Szegedy et al. \cite{szegedy2014intriguing} in 2014.
They have shown the existence of easily craftable imperceptible vulnerabilities in image classifiers. These vulnerabilities have taken the form of perturbations, imperceptible to the human eye, that can significantly alter a model's prediction on an otherwise benign input.
The low requirements for mounting adversarial attacks make it a prime vulnerability.
Prior work in the domain of adversarial examples suggests a notable advantage for attackers, with the majority of attacks capable of, to some degree, impeding model performance. 
Among the existing defenses, few techniques have demonstrated a degree of robustness when it comes to safeguarding machine learning models, for example, adversarial training methods \cite{madry2019deep} and ensemble approaches \cite{tramer2020ensemble} (and their combination).
However, even in their strongest forms and under strict threat models that greatly limit the attacker's capabilities, current defenses still fail to provide adequate protection. To remedy this, a branch of research on provably robust defenses emerged \cite{cohen2019certified}.
However, in practice, this is not yet a feasible approach since it increases the computational cost of model inference by several orders of magnitude. Since provable defenses are impractical, a majority of the field has focused on improving the quality of empirical defenses.
Nevertheless, the absence of a reliable evaluation standard or assessment method for overall performance has resulted in an ongoing competitive cycle between attackers and defenders based on empirical experiments. 
A complete theoretical framework is necessary to create an evaluation standard and shift this competitive cycle into a manageable challenge for adversarial example research. 

In our work, we focus on the image classification domain and aim to address a significant challenge: the current lack of emphasis on rigorous assessment of adversary knowledge in threat models. 
Adversary knowledge is one of the key components of a well-defined threat model and has long been overlooked and neglected by both attack and defense works. 
We study the current lack of formalization of adversary knowledge in adversarial example research and how it affects attack comparability and performance.
Defending against attackers with ill-defined capabilities places an additional burden on the defender. Whereas attacking with ill-defined knowledge hinders reproducibility and comparability.
We develop this formalization as a theoretical framework to model various adversarial example attacks. 
We conduct a thorough review of recent attacks in the image classification domain, offering an up-to-date overview of the current landscape of adversaries in this specific context \cite{zhang2022imagenet, Liu_Lu_Xiong_Zhang_Xiong_2023, Shamshad_Naseer_Nandakumar_2023, Zhong_2022_CVPR, Zou_Duan_Li_Zhang_Pan_Pan_2022, Wei_Zhao_2023, Wei_Chen_Wu_Jiang_2023, Weng_Luo_Lin_Li_Zhong_2023, Zhang_Tan_Sun_Zhao_Zhang_Li_2023,
Tao_Wu_2022, 10147340, Deng_Xiao_Li_He_Wang_2023, Ge_Shang_Liu_Liu_Wang_2023, Ge_Shang_Liu_Liu_Wan_Feng_Wang_2023, Liu_Ge_Zhou_Shang_Liu_Jiao_2023, Wang_Yang_Feng_Sun_Guo_Zhang_Ren_2023, Tramer_2022, Byun_Kwon_Cho_Kim_Kim_2023, Chen_Yin_Chen_Chen_Liu_2023,
 Li_Guo_Yang_Zuo_Chen_2023, Li_Guo_Zuo_Chen_2023, Qin_Fan_Liu_Shen_Zhang_Wang_Wu_2022, Yang_Lin_Zhang_Yang_Zhao_2023, Yang_Lin_Li_Shen_Zhou_Wang_Liu_2023, Qi_Huang_Panda_Henderson_Wang_Mittal_2023, Pintor_Demetrio_Sotgiu_Demontis_Carlini_Biggio_Roli_2022, Pomponi_2022, Li_decisionbased2022, liu2022practical, liu2022transferable, Apruzzese_Anderson_Dambra_Freeman_Pierazzi_Roundy_2023, 
 Agarwal_Ratha_Vatsa_Singh_2022, Aldahdooh_Hamidouche_Fezza_Deforges_2022, Altinisik_Messaoud_Sencar_Chawla_2023, Byun_2022_CVPR, chenImageRetrieval2023, Chen_Liu_2023, Crecchi_Melis_Sotgiu_Bacciu_Biggio_2022, Dai_Feng_Chen_Lu_Xia_2022, Dunmore_Jang-Jaccard_Sabrina_Kwak_2023, Dyrmishi_Ghamizi_Simonetto_Traon_Cordy_2023, Freiesleben_2022, 
 Gong_Deng_2022, Gubri_Cordy_Papadakis_Traon_Sen_2022, Hernandez-Castro_Liu_Serban_Tsingenopoulos_Joosen_2022, Hu_Li_Yuan_Cheng_Yuan_Zhu_2022, Jun-hua_Ye-xin_Chuan-lun_Jun-yang_Xing-yu_Zhi-song_2022, Lee_Kim_2023, Li_Cheng_Hsieh_Lee_2022, Liang_Wu_Hua_Zhang_Xue_Song_Xue_Ma_Guan_2023, Liang_He_Zhao_Jia_Li_2022, Lin_Hsu_Chen_Yu_2022, Liu_Zhang_Mo_Xiang_Li_Cheng_Gao_Liu_Chen_Wei_2022, Luo_2022_CVPR, midtlidMASSA, Namiot_Ilyushin_2022, Pang_Zhang_He_Dong_Su_Chen_Zhu_Liu_2022, Pawelczyk_Agarwal_Joshi_Upadhyay_Lakkaraju_2022, Pinot_Meunier_Yger_Gouy-Pailler_Chevaleyre_Atif_2022,  qi2022, Qian_Huang_Wang_Zhang_2022, Ruiz_Kortylewski_Qiu_Xie_Bargal_Yuille_Sclaroff_2022, Sajeeda_Hossain_2022, Singh_Awasthi_Urvashi_2022, Tuna_Catak_Eskil_2022, Vos_Verwer_2022, Sheatsley_Hoak_Pauley_McDaniel_2022, Gubri_Cordy_Traon_2023, Li_Yu_Huang_2023, 10219877, wang_2022_vit, pmlr-v162-yamamura22a, long2022frequency, chen2022shape, chen2023diffusion, chen2023contentbased, dai2023advdiff, liu2023diffprotect, lin2023diffusion, diffpure, bose2020adversarial, Wang_2021_ICCV, aigan2021}
.
This formalization allows us to define and categorize attacks and their associated threat models, in turn, laying a foundation for defenses. Additionally, combined with our survey of attacks, we measure 
the effect of the information available to an attacker on an attack's performance. We compare various threat models and confirm, with falsifiable results, the generally held belief that certain categories of information, like information related to 
the attacked classification model, have a disproportionate effect on attack performance. However, we also find that a lack of information in that category, in the case of transferable attacks, can be compensated by the use
of additional information such as information related to the training data or the training process, to yield attacks that are almost as potent.

\section{Definitions}
We let $\mathcal{I} \subset \mathbb{R}^*$ be the set of all possible inputs,  $\mathcal{L} \subset \mathbb{R}^* \cup \{\bot\}$ the set of all possible labels, and the ground-truth function $gt: \mathcal{I} \rightarrow \mathcal{L}$ as the function that assigns its true label to an input sample.
\begin{definition}[Adversarial Example]\label{def:new_adversarial_example}
    An input sample $i \in \mathcal{I}$ and its associated label $gt(i) \in \mathcal{L}$  is said to be adversarial if it was specifically crafted to successfully trigger a learned model to output an incorrect answer.
\end{definition}
\begin{definition}[Grounded Adversarial Example]\label{def:grounded}
    An input sample $i \in \mathcal{I}$ and its associated label $gt(i) \in \mathcal{L}$ is said to be a \underline{grounded adversarial example} if it was created using a benign sample $x \in \mathcal{I}$ and its associated label $gt(x) \in \mathcal{L}$.
\end{definition}
\begin{definition}[Targeted/Untargeted Adversarial Example]\label{def:targeted_untargeted_adersarial_example}
    An input sample $i \in \mathcal{I}$ and its associated label $gt(i) \in \mathcal{L}$ is said to be a \underline{targeted adversarial example} if for a fixed chosen label $c \in \mathcal{L}, c\neq gt(i)$, it triggers a learned model to output label $c$. However, if it triggers a learned model to output a label $l \neq gt(i), l \in \mathcal{L}$, we say it is \underline{untargeted}.
\end{definition}

\section{Previous Work}
We formalize adversarial example attacks with an emphasis on the information and the capabilities the adversary has access to perform their attack. There has been a lack of a proper theoretical framework that is based on proper security principles. In turn, this has been reflected in work whose impact is limited due to the inherent flaws in their security assumptions. 

Previous systematization of knowledge works \cite{papernotsok, carlini_and_wagner2017} have looked at the adversarial example problem from various angles. Papernot et al. \cite{papernotsok} took a holistic approach and proposed a thorough overview of the possible 
threat models as well as the training and inference process in an adversarial setting. They emphasized a clear distinction between the different domains (physical vs. digital) while describing the overall ML attack surface. 
This was reiterated more recently by Lin et al. \cite{Lin_Hsu_Chen_Yu_2022} where the authors present an attack on the physical domain specifically. In contrast, Byun et al. \cite{Byun_2022_CVPR} took advantage of the properties of the physical domain by projecting adversarial examples onto 
virtually 3D-rendered physical objects (like a mug) to improve the transferability of their attack. Another approach was the one by Carlini \& Wagner \cite{carlini_and_wagner2017}. In their paper, they present generating adversarial examples as a constrained optimization problem. 
They first cover a set of existing attacks at the time and from their formalization, they introduced three new attacks adapted each for a specific norm distance metric ($L_0$-norm, $L_1$-norm, and $L_2$-norm).
However, we found that these works fail to properly cover adversarial information and capabilities. This was not a significant problem back in 2017 and 2018 when most of the literature focused on white-box attacks,
attacks where the adversary has immense knowledge and capabilities and therefore covering the minutia is not as important. \\

Recently, however, the field has experienced a shift toward more realistic and usable attacks. This is reflected in the emergence of unrestricted adversarial examples \cite{chen2023contentbased, chen2021unrestricted, chen2023diffusion} as well as threat models 
like the black-box and no-box (transferable) threat models \cite{zhang2022imagenet, Byun_2022_CVPR, Gubri_Cordy_Papadakis_Traon_Sen_2022, long2022frequency, Wang_2021_ICCV, wang_2022_vit}.
A work that took the first step toward rationalizing this shift is the work of Gilmer et al. \cite{gilmer2018motivating}. One of the core messages of their work is that, at the time,
research on adversarial examples was predominantly focused on abstract, hypothetical scenarios that lacked direct relevance to any specific security issues. This still partially holds up to this day. They did mention the question 
of the attacker's knowledge but did not go into depth beyond the already common concepts of white-box and black-box.  Papers on defense techniques have also not provided a comprehensive account of attackers' capabilities 
and constraints that would be applicable in real-world security scenarios. However, they presented a wide-reaching list of salient situations that represent at their core the action space of the attacker. The action space of an attacker can be defined as 
the set of potential actions or strategies at the disposal of an adversary. This encompasses the various methods and choices available to the attacker when attempting to compromise a system, exploit vulnerabilities, or achieve their malicious goals during an attack. \\

In this work, we propose to delve deeper into the knowledge that an attacker has when performing an attack. This goes beyond the knowledge of the machine learning system that they are trying to attack. It includes 
access to data and computing resources, as well as knowledge of the code and hyperparameters used to set up and train said machine learning system.

\section{Formalization}
\raggedbottom

A key element to differentiate between inference time attacks is the action space that is available to the attacker. Real-world practical restrictions are translated into mathematical constraints on the generated adversarial examples.
Gilmer et al. \cite{gilmer2018motivating} first define the action space of the adversary for adversarial example attacks. In their paper, they identify five salient situations: indistinguishable perturbations, content-preserving perturbations, non-suspicious input, content-constrained input and unconstrained input. Historically, adversarial example attacks have been studied under the indistinguishable and content-preserving perturbation settings, but recently have been extended to the non-suspicious and content-constrained input settings.  

To propose a unified theoretical framework that can accommodate every salient situation, we introduce an object that is already well-defined in the field of cryptography: the distinguisher.

Our fundamental challenge is rooted in simulating human-like image perception, a problem lacking a well-established mathematical framework. 
Consequently, we must define our own setting, complete with its corresponding distinguisher.
The context we are operating within assumes human-like capabilities for the distinguisher, for example, for indistinguishable perturbations, it is so human reviewers would not be able to distinguish that the image is adversarial. Therefore, we need to introduce a new class of distinguishers, the \textbf{human-like distinguisher} that we define as \underline{a distinguisher whose capabilities match that of a human or }\\
\underline{any human-made real-world system}. Current existing metrics ($l_p$-norm, similarity, quality...) all fall under this category and are used as an approximation of human-like capabilities for their salient situation.

We can now represent a salient situation as a distinguisher $D: \mathcal{I} \rightarrow \{0,1\}$ who attempts to output 0 on benign samples and 1 on adversarial samples, it is meant to capture the detection capabilities in place in the system under attack.
We can use this formalization of salient situations and provide a proper formalization of the notion of \textit{indistinguishability} for any setting, not just the indistinguishable perturbation setting.

\begin{definition}[Indistinguishability]\label{def:indistinguishability}
    Assuming a salient situation with distinguisher $D: \mathcal{I} \rightarrow \{0,1\}$, let $A, B \subset \mathcal{I}$ where $A$ is the set of all adversarial examples and $B$ is the set of all benign samples, we say an input $x \in \mathcal{I}$ 
    is \underline{indistinguishable} if we have:\\
    \begin{gather}
        \zeta(n) = |\underset{x \leftarrow A}{\textnormal{Pr}}[D(x) = 1] - \underset{x \leftarrow B}{\textnormal{Pr}}[D(x) = 1]| = 0
    \end{gather} 
\end{definition}

However, this definition is too restrictive as in practice $n$ (and therefore the number of possible adversarial examples) is too large to obtain $\zeta(n) = 0$ and should only serve as a target for the community to strive for
rather than a strict requirement. Instead, for practical purposes, we can use a loosened notion of \emph{stealth} to bind adversarial examples.

\begin{definition}[Stealth]\label{def:stealth}
    If $\zeta(n)$ (from Equation \ref{def:indistinguishability}) is $\zeta(n) \in \mathcal{O}(\frac{log(n)}{n})$, $|A \cup B| = n$, then we say that the associated input $x \in \mathcal{I}$ is \underline{stealthy}.
\end{definition}

\subsection{Categorization}
\subsubsection{Information Extraction Oracles}

To formalize the information an adversary has access to when performing an attack, we need a generic representation of information. Since precisely quantifying and defining information is quite difficult,
we avoid it and instead present a generic method for representing information acquired by the adversary. To that effect, we introduce Information Extraction Oracles (IEOs).
IEOs serve as a translation tool that allows an attacker to properly defined information in a threat model. They can be used to represent the access to information available to the attacker (through any of its attack capabilities). We now introduce an improved method for presenting threat models in a standardized and formalized fashion.

\begin{definition}[Information Extraction Oracle]\label{def:ieo}
    An Information Extraction Oracle $\mathcal{O}: \{0, 1\}^* \rightarrow \{0, 1\}^*$ is a stateful oracle machine that can take in a query and outputs information in a binary format.
\end{definition}

We define the set of all information extraction oracles as $\mathbb{O}$.
Instead of describing attacker capabilities with vague terminology, we can now use information extraction oracles to precisely and accurately capture the attacker's knowledge and capabilities.
As mentioned in Definition \ref{def:ieo}, information extraction oracles are stateful. The adversary can extract the oracle's state using the $State$ function.

\begin{definition} \label{def:state}
    We define the function $State: \mathbb{O} \rightarrow \mathbb{R}^*$ that takes in as input an Information Extraction Oracle and outputs its current state as a set of real numbers.
    If the input oracle does not have a state, it returns $\emptyset$.
\end{definition}

To accurately represent relationships between various threat models (white-box vs. black-box for example), we need to introduce a relational structure to information extraction oracles. To do so, we define the domination $\sqsubset$ operator on IEOs.

We define the different symbols we use:
\begin{itemize}[noitemsep, topsep=0pt]
    \item $\{\}$ represents an unordered set
    \item $[]$ represents an ordered set
    \item $\cdot$ represents what we call an element, which we define as anything that is not a set.
\end{itemize}

\begin{definition}[Information Extraction Oracle domination operator]\label{def:sqsubset}
    We define the operator $\sqsubset$ for information extraction oracles over their outputs in the following way: 

    Let $\mathcal{O}_1(a)=x$ and $\mathcal{O}_2(a)=y$. We have three base cases:
    \begin{enumerate}[noitemsep, topsep=0pt]
        \item 
        \begin{enumerate}[noitemsep, topsep=0pt]
            \item $x$ and $y$ are both sets. Then $\mathcal{O}_2 \sqsubset \mathcal{O}_1$ if $\forall a \in \mathcal{I}$, $x \subset y$.
            \item $x$ is an element and $y$ is an unordered set. Then $\mathcal{O}_2 \sqsubset \mathcal{O}_1$ if $\forall a \in \mathcal{I}$, $x \in y$.
            \item $x$ is either an unordered set or an element and $y$ is an element. Then $\mathcal{O}_2 \sqsubset \mathcal{O}_1$ if there exists a probabilistic polynomial-time (PPT) function $f$ s.t. $\forall a \in \mathcal{I}$, $f(x) = y$.
        \end{enumerate}
    Using these base cases, we can expand to the three ordered set cases.
    \item
    \begin{enumerate}[noitemsep, topsep=0pt]
        \item $x=[x_1, \dots, x_n]$ is an ordered set and $y$ is either an unordered set or an element ($n$ is a positive integer). Then $\mathcal{O}_2 \sqsubset \mathcal{O}_1$ if $(y \sqsubset x_1) \vee  \dots \vee (y \sqsubset x_n)$.
        \item $x$ is an unordered set or an element and $y=[y_1, \dots, y_k]$ is an ordered set ($k$ is a positive integer). Then $\mathcal{O}_2 \sqsubset \mathcal{O}_1$ if $(y_1 \sqsubset x) \wedge  \dots \wedge (y_k \sqsubset x)$.
        \item $x=[x_1, \dots, x_n]$ is an ordered set and $y=[y_1, \dots, y_k]$ is an ordered set ($n$ and $k$ are positive integers). Then $\mathcal{O}_2 \sqsubset \mathcal{O}_1$ if $\forall i \in \{1, \dots, k\}, \exists j \in \{1, \dots, n\}$ s.t. $y_i \sqsubset x_j$.
    \end{enumerate}
\end{enumerate}
\end{definition} 

We let $\not\sqsubset$ be the not operator for $\sqsubset$. Meaning when not $A \sqsubset B$, then $A \not\sqsubset B$. 

When an attacker has access to multiple oracles $\mathcal{O}_A, \mathcal{O}_B$ where a domination relation cannot be established, we describe the resulting combined oracle in Definition \ref{def:ieo_combination}. 
\begin{definition}[Information Extraction Oracle Combination]\label{def:ieo_combination}
    Given two IEOs $\mathcal{O}_A, \mathcal{O}_B$ for which neither $\mathcal{O}_A \sqsubset \mathcal{O}_B$ nor $\mathcal{O}_B \sqsubset \mathcal{O}_A$ is true. Then we define the combined oracle $\mathcal{O}_{A\&B}$ as follows:
    $\mathcal{O}_{A\&B}(x) = [\mathcal{O}_A(x), \mathcal{O}_B(x)]$
\end{definition}

\subsubsection{Information Categories}

An IEO provides an interface for the attacker to both the knowledge of the defender (for example, information about the target model) and any additional external knowledge possessed by the attacker 
(for example, an additional dataset that is disjoint from the defender's training set but is drawn from the same distribution). 
We can then categorize the different aspects of a threat model into different IEOs that allow the attacker access to the associated information. We used all the threat models
we observed in the literature to build this categorization to have a complete representation of the threat models in the field and a meaningful categorization.

We can classify the information involved in an adversarial example attack into three distinct types: information \textbf{used} by the defender, \textbf{generated} by the defender (for example model parameters $\theta_M$) and any information that is \textbf{publicly accessible} (data, pre-trained models \dots).

Within those classes, we identified the salient categories that combined can be used to reconstruct any threat model: \textbf{model}, \textbf{data}, \textbf{training} and \textbf{defense} information.

\subsubsection{Information Hasse Diagrams}

Now that we have defined broad information categories for attackers, we can introduce a structure that will allow for meaningful comparisons between sets of assumptions within each category. 

We make use of the $\sqsubset$ operator that we defined in Definition \ref{def:sqsubset} to build Hasse diagrams. These diagrams are usually used to visually represent sets ordered by inclusion. We design our Hasse Diagrams to visually represent our oracles ordered by $\sqsubset$.
For each category, we provide the associated Hasse diagram. In order theory, a Hasse diagram is supposed to represent finite partially ordered sets, however, we extend it to infinite partially ordered sets by allowing formulaic set descriptions (which can expand to generate infinite but partially ordered sets).
For simplicity's sake, when we describe information extraction oracles, we assume that the inputs they are queried with are in $\mathcal{I}$ (in the image domain), otherwise, they return an empty ordered set $[]$.  We summarize all the behavior of the information extraction oracles we define in Table \ref{tab:oracles}.
\paragraph{Model information Hasse diagram: \newline} 
\quad For model-related information, it can either be \textit{static} or \textit{query-based}. 
For static information, we have:
\begin{itemize}[noitemsep, topsep=0pt]
    \item \underline{Model parameters} $\theta_M$ (white-box). 
    Let $\mathcal{O}_M$ be its information extraction oracle. 
    \item Model architecture: either the exact model architecture or a set of possible architectures.
    \begin{itemize}[noitemsep, topsep=0pt]
        \item \underline{Exact model architecture} $\phi_M$.
        Let $\mathcal{O}_A$ be its information extraction oracle.
        \item \underline{Set of possible architectures}: Let $\mathcal{O}_{SPA}$ be its information extraction oracle.
    \end{itemize}
\end{itemize}
For query-based information, we have:
\begin{itemize}[noitemsep, topsep=0pt]
    \item Query-access to \underline{model scores} (model logit probabilities or values on queried inputs).
    Let $\mathcal{O}_S$ be its information extraction oracle. 
    \item Query-access to \underline{model labels} (predicted class on queried inputs). Let $\mathcal{O}_L$ be its information extraction oracle.
\end{itemize}

\begin{table}[h]
\centering
\resizebox{\columnwidth}{!}{
\begin{tabular}{ccc}
Oracle & Output & State \\ \toprule
$\mathcal{O}_M$      &   $[\theta_M, x]$    & $[]$    \\
$\mathcal{O}_S$ & $[M(x), x]$ & \begin{tabular}[c]{@{}c@{}}number of queries \\left $[k]$\end{tabular} \\
$\mathcal{O}_L$ & $[argmax(M(x)), x]$ & \begin{tabular}[c]{@{}c@{}}number of queries \\left $[k]$\end{tabular} \\
$\mathcal{O}_A$  & $[\phi_M, x]$ & $[]$ \\
$\mathcal{O}_{SPA}$ & \begin{tabular}[c]{@{}c@{}}$[\{\phi_0, \phi_1, \dots, \phi_k\}, x]$, \\$k \in \mathbb{Z}^+$ and for some \\$i \sim \mathcal{U}(0, k)$, $\phi_i = \phi_M$\end{tabular} & $[]$ \\
\midrule
$\mathcal{O}_{\mathcal{D}}$ & $[\mathcal{D}, x]$ & $[]$ \\
$\mathcal{O}_{\mathcal{D'}}$ & $[\mathcal{D'}, x]$ & $[]$ \\
$\mathcal{O}_{\mathcal{E}}$ & $[\mathcal{E}, x]$ & $[]$ \\
\midrule
$\mathcal{O}_{Train}$ & $[Train, x]$ & $[]$ \\
$\mathcal{O}_{T_i}$ & $[T_i, x]$ & $[]$ \\
\midrule
 $\mathcal{O}_{FA}$ & $[\rho, \varrho, x]$ & $[]$ \\
 $\mathcal{O}_{PA}$ & \begin{tabular}[c]{@{}c@{}}$[\rho, \{\varrho_1, \dots, \varrho_k\}, x]$, \\$k \in \mathbb{Z}^+$ and for some \\$i \sim \mathcal{U}(0, k)$, $\varrho_i = \varrho$ \end{tabular}& $[]$ \\ 
 $\mathcal{O}_{SPD}$ & \begin{tabular}[c]{@{}c@{}}$[\{\rho_0, \dots, \rho_k\}, \{\varrho_1, \dots, \varrho_k\}, x]$ \\ $k \in \mathbb{Z}^+$ and for some \\$i \sim \mathcal{U}(0, k)$, $\rho_i = \rho$ and $\varrho_i = \rho$\end{tabular} & $[]$ \\
\bottomrule
\end{tabular}
}
\caption{Oracle Table}
\label{tab:oracles}
\end{table}
We can then order their associated information extraction oracles in a Hasse Diagram in Figure \ref{fig:model_hasse_diagram}.

\begin{figure}[h]
    \centering
    \includegraphics[width=6.2cm, height=8.2cm]{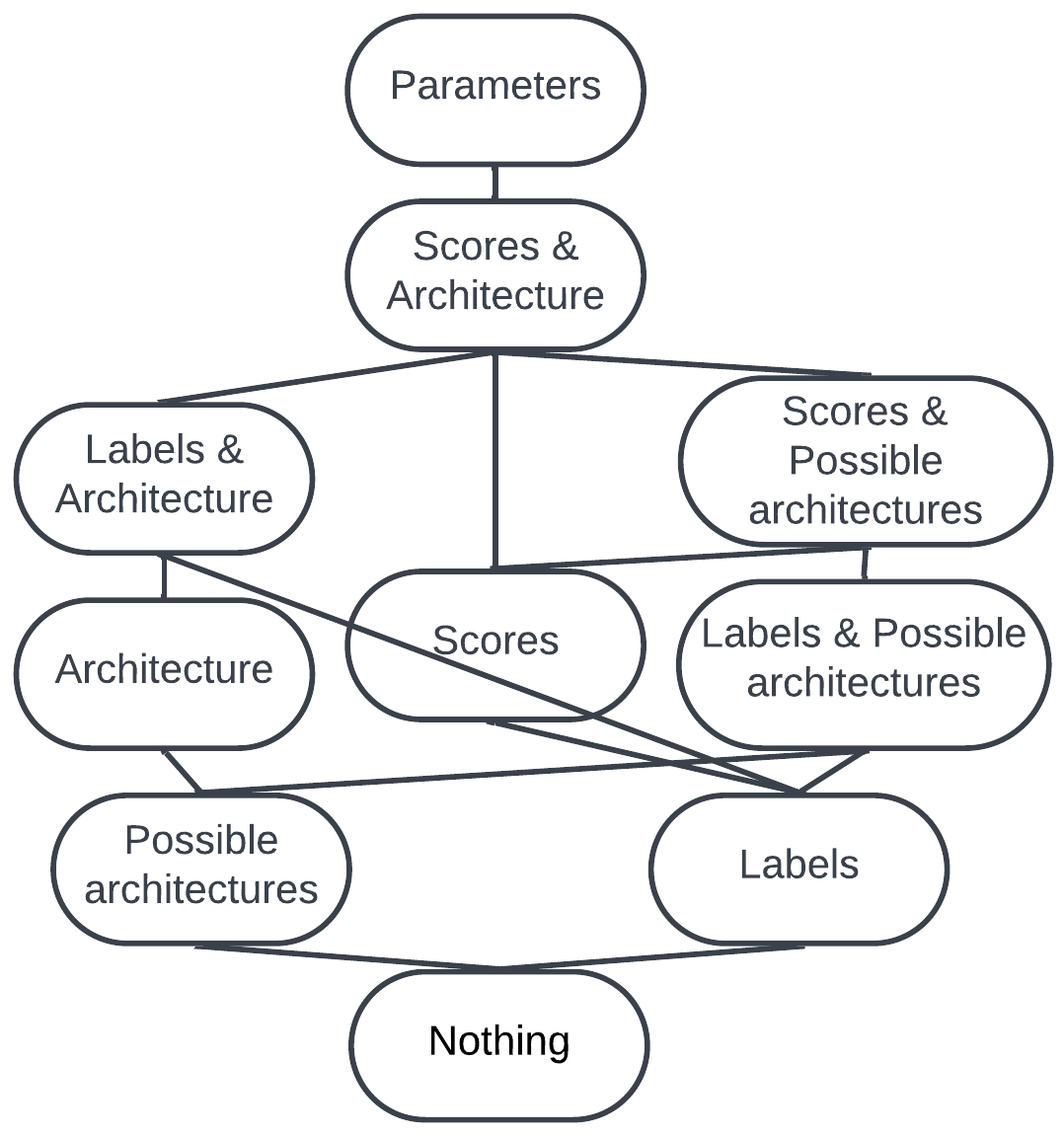}    
    \caption{Model Oracle Hasse Diagram}\label{fig:model_hasse_diagram}
\end{figure}

\begin{theorem}\label{theorem:model_hasse_diagram}
    Figure \ref{fig:model_hasse_diagram} holds under the $\sqsubset$ ordering.
\end{theorem}
\noindent
We refer to Appendix \ref{app:proofs} for all of the proofs.

\paragraph{Data Information Hasse Diagram}
We identify three sets of data. The \underline{training data} $\mathcal{D} \subset \mathcal{I} \times \mathcal{L}$ that was used to train the target model $M$ and its information extraction oracle $\mathcal{O}_{\mathcal{D}}$.
Any other data samples $\mathcal{D}' \subset \mathcal{I} \times \mathcal{L}$ (and it's IEO $\mathcal{O}_{\mathcal{D'}}$) that were drawn from the \underline{same distribution} as $\mathcal{D}$ but were not in the training data. 
Finally, any \underline{other data} samples $\mathcal{E} \subset \mathcal{I} \times \mathcal{L}$ (and it's IEO $\mathcal{O}_{\mathcal{E}}$) that were not drawn from the same distribution as $\mathcal{D}$ and $\mathcal{D'}$.
In this case, creating the Hasse diagram is relatively straightforward because we only need to contemplate all feasible power set combinations of oracles using definition \ref{def:ieo_combination}. 

\begin{figure}[h]
    \centering
    \includegraphics[width=5.5cm, height=5.5cm]{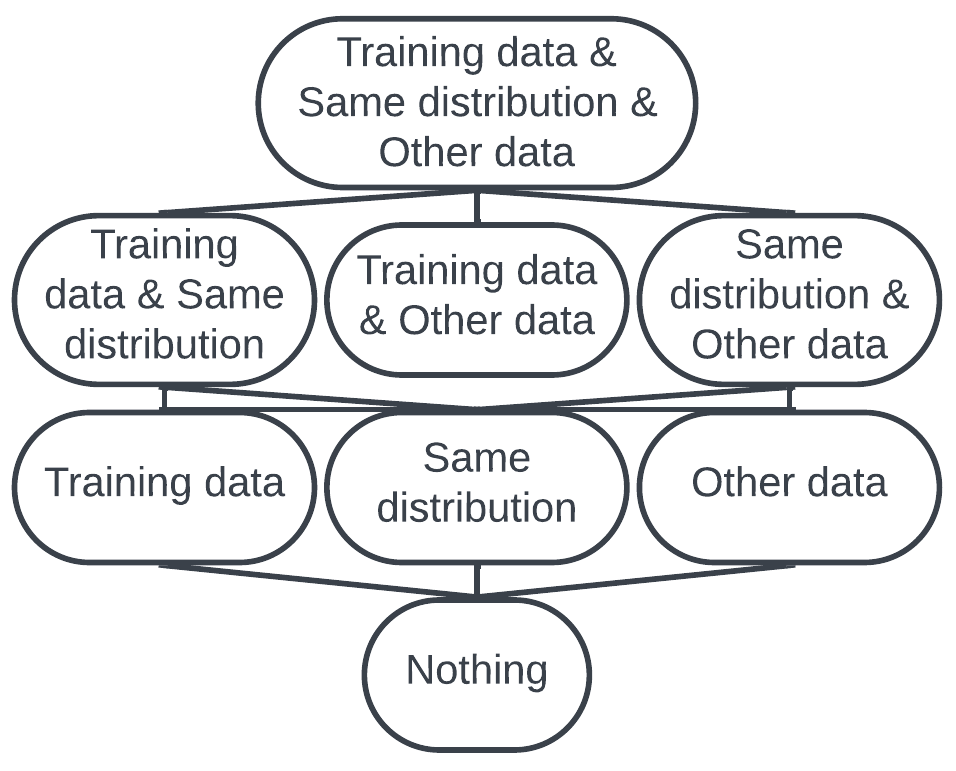}  
    \caption{Data Oracle Hasse Diagram}\label{fig:data_hasse_diagram}
\end{figure}

\begin{theorem}\label{theorem:data_hasse_diagram}
    Figure \ref{fig:data_hasse_diagram} holds under the $\sqsubset$ ordering.
\end{theorem}

\paragraph{Training Information Hasse Diagram}\label{train_info_hasse_diagram}

We describe a systematized method for identifying a training function's known elements. First, we must define exactly what we mean by training algorithm or function.

\begin{definition}[Train function]\label{def:train}
    We define a training function $Train: \{0,1\}^* \rightarrow \mathcal{M}$ as the function representing the algorithm used by the defender to
    generate the target model $M$.
\end{definition}
The $Train$ function takes in the algorithm's hyper-parameters as well as any other information needed to train the model. 
We can articulate known information about the training algorithm as one of the following \underline{training information sets}:\\ $T = \{Train' | g_{Train}(Train') = True\}$ where $g_{Train}$ is a function that returns $True$ if a particular condition in the training algorithm with respect to the original $Train$ function is satisfied (e.g. the training algorithm uses cross-entropy loss) and $False$ otherwise. We can then create two types of information extraction oracles. This is the information that can be used (combined with data) to train/assume access to pre-trained models. The first one is $\mathcal{O}_{Train}$, where $Train$ is the \underline{original $Train$ function} used by the defender. 

 The other is when the attacker has access to one of the sets $T$. There are infinitely many of them, therefore 
we will just define a generic oracle for them: $\mathcal{O}_{T_i}$ for any integer $i$ where $T_i$ is any one of the sets $T_i = \{Train' | g^i_{Train}(Train') =True\}$ we previously described. 

The oracles generated from this generic oracle can then be ordered using $\sqsubset$ following the usual rules defined in 
Definition \ref{def:sqsubset}.

\begin{figure}[h]
    \centering
    \includegraphics[width=2.7cm, height=4.1cm]{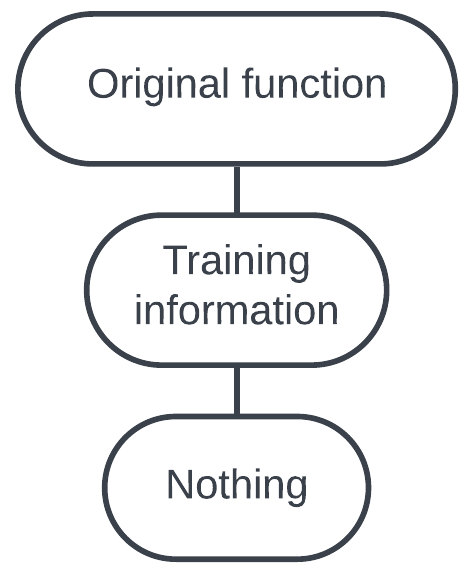}  
    \caption{Train Oracle Hasse Diagram}\label{fig:train_hasse_diagram}
\end{figure}

\begin{theorem}\label{theorem:train_hasse_diagram}
    Figure \ref{fig:train_hasse_diagram} holds under the $\sqsubset$ ordering.
\end{theorem}

\paragraph{Defense Information Hasse Diagram}

Defense information is any defense mechanisms put in place by the defender on top of the distinguisher $D$. We can partition it as follows:

\begin{itemize}[noitemsep, topsep=0pt]
    \item \underline{Full awareness} of the defense and its parameters (similar to white-box for model information). We represent this as an algorithm $\rho$ and its parameters $\varrho$. Let $\mathcal{O}_{FA}$ be its information extraction oracle.
    \item \underline{Partial awareness} of the defense and its parameters: knowledge of the defense $\rho$ but not the specific parameters of the instance. Let $\mathcal{O}_{PA}$ be its information extraction oracle. 
    \item \underline{A set of potential defenses} that can be obtained by having partial insider information. Let $\mathcal{O}_{SPD}$ be its information extraction oracle.
\end{itemize}

\begin{figure}[h]
    \centering
    
    \includegraphics[width=1.84cm, height=5.5cm]{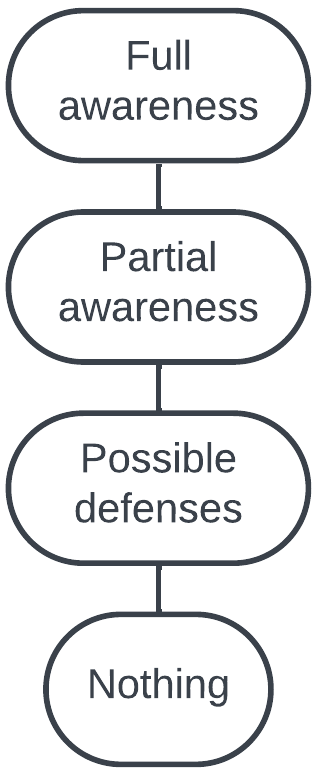}
    \caption{Defense Oracle Hasse Diagram}\label{fig:defense_hasse_diagram}
\end{figure}

\begin{theorem}\label{theorem:defense_hasse_diagram}
    Figure \ref{fig:defense_hasse_diagram} holds under the $\sqsubset$ ordering.
\end{theorem}

\subsection{Adversarial Example Game}
Gilmer et al. \cite{gilmer2018motivating} were also among the first to present the problem of adversarial examples as a security game. This contrasts with Carlini \& Wagner \cite{carlini_and_wagner2017} and their optimization-based representation. 
This emphasized the security nature of the problem. They also address in their paper the notion of game sequence. It is the player order and whether the game is repeated. This was not particularly addressed by Carlini \& Wagner \cite{carlini_and_wagner2017} but is key to a practical understanding of the problem. It relates the problem of adversarial examples to those of standard cryptographic security games. Later, Bose et al. \cite{bose2020adversarial} expanded upon the game nature of adversarial
examples while relating it to its underlying min-max optimization nature. They use this expansion to develop an attack in what they call the NoBox (also called transferable) setting and to bridge the gap between the theory and the empirical in more
demanding and realistic threat models.
We present an updated, concise and comprehensive security game for adversarial examples
that conforms with the security and machine learning communities.

\subsubsection{Definitions}

First, we define the following: $\varnothing$ is the termination symbol. $\mathcal{O}^{a,b}_{X}$ is an input information extraction oracle about the data granted to the adversary for an attack. We superscript $\mathcal{O}^{a,b}_{X}$ with a bit $a$ that represents whether the attack is grounded or not (0 for grounded and 1 for not-grounded). We also superscript it with a bit $b \in \{0,1\}$ where if the attack is untargeted then $b = 0$ and if it is targeted then $b=1$.
It returns the start input sample $x \in \mathcal{I}$ and its associated ground-truth label $y \in \mathcal{L}$ if the attack is grounded.
In the case of a targeted attack, it also returns the target label $y_t \in \mathcal{L}$.
$\mathcal{O}_{Dist}$ is an information extraction oracle that returns a distinguisher $D'$ that the attacker can utilize to verify that the generated adversarial examples are within the restrictions of the salient situation. By default, $D'=D$ where $D$ is the actual distinguisher used by the defender. However, in situations where only partial information 
    is known about the distinguisher $D$ or the associated salient situation, $D'$ can be used to capture the attacker's knowledge of $D$.

We can now define the key functions that we will use in our security game: 

\begin{definition}[$AdvGen$]\label{def:advgen}
    $AdvGen: \mathbb{O} \rightarrow \mathcal{I} \times \mathcal{L}$. \\
    $AdvGen(O)$ on input set of information oracles $O \subset \mathbb{O}$ (including the input information oracle $\mathcal{O}^{a,b}_{X}$), outputs the following:
    \begin{gather}
        AdvGen(O) = [x', y']
    \end{gather}
    where $x' \in \mathcal{I}$ is the adversarial example and $y' \in \mathcal{L}$ is its ground-truth label.

\end{definition}

Additionally, we also define two other functions that are vital to our security game: $Evaluate$ and $Classify$. As mentioned in Gilmer et al.'s work \cite{gilmer2018motivating}, an attacker can have
different goals, therefore we use the $Evaluate$ function to represent an attack's success with respect to the attacker's goal. (Either a targeted or an untargeted (Definition \ref{def:targeted_untargeted_adersarial_example}) attack or any other goal an attacker might have). 

\begin{definition}[$Classify$]\label{def:classify}
    $Classify(M, x)$ on input model $M$ and input sample $x \in \mathcal{I}$ returns the predicted label of $x$ using the classifier $M$ while executing any inference-time pre-processing or defense mechanisms the defender has in place.    
\end{definition}

\begin{definition}[$Evaluate$]\label{def:evaluate}
    We have two evaluation functions, one for untargeted attacks $Evaluate_0$ and one for targeted attacks $Evaluate_1$ (same as $\mathcal{O}^{a,0}_{X}$ and $\mathcal{O}^{a,1}_{X}$). They are defined as the following:
    \begin{gather}
        Evaluate_0(y, r) = \mathbb{I}[r \neq y] 
    \end{gather}
    and
    \begin{gather}
        Evaluate_1(y_t, r) = \mathbb{I}[r = y_t]
    \end{gather}
    Where $y$ is the ground-truth label, $r$ is the predicted label and $y_t$ is the target label for a targeted attack.
    
\end{definition}

Now that we have all the components of the game ready, we can define it in Figure \ref{alg:game_diagram}.
This game has two participants, the attacker and the defender. 

\begin{figure}[h]
\resizebox{\columnwidth-20pt}{!}{
\begin{minipage}{\columnwidth}
\pseudocodeblock[linenumbering, space=keep]{
\textbf{ Attacker} \< \< \textbf{ Defender} \\[0.1\baselineskip ][\hline ]
(O, AdvGen, Evaluate_b) \< \< (D, Train, Classify) \\[0.1\baselineskip ][\hline ]
\<\< M \gets Train(\cdot) \\
l \gets \{\} \\
i \gets 0 \\
\mathcal{O}_{Dist} \gets O[\cdot] \\
D' \gets \mathcal{O}_{Dist} \\ 
x', d \gets \emptyset, 1 \\
x_i' \gets x' \\
\textbf{While } (d=1) \wedge (x_i' \notin l) \\
\textbf{Do }\\
x_i', y_i' \gets AdvGen(O) \\
l \gets l \cup \{x_i'\} \\
d \gets D'(x_i', y_i', \mathcal{O}^{a,b}_{X}) \\
i \gets i + 1 \\
\textbf{Done} \\
x' \gets x_i' \\
\textbf{If } d=1 \textbf{ Then Return } 0 \\
\textbf{Else }\< \sendmessageright*[1.4cm]{x'} \\
\<\< d' \gets D(x') \\
\<\< \textbf{If } d' = 1 \textbf{ Then } r \gets \varnothing \\
\<\< \textbf{Else } r \gets Classify(M, x') \\
\< \sendmessageleft*[1.4cm]{r} \\
\textbf{If } r = \varnothing \textbf{ Then Return } 0 \\
\textbf{If } b=0 \textbf{ Then}\\
\textbf{Return } Evaluate_0(y, r) \\
\textbf{Else } \\
\cdot, \cdot, y_t \gets \mathcal{O}^{a,1}_{X} \\
\textbf{Return } Evaluate_1(y_t, r)
}
\end{minipage}
}

\pseudocodeblock[space=keep]{
    \gets \text{ is the assignment operator.}= \text{ is the equality comparison}\\ \text{operator. }
    \cdot \text{ is used as a placeholder for an index.}
}
\caption{Adversarial Example Game Diagram Algorithm}\label{alg:game_diagram}
\end{figure}

\subsubsection{Measuring Success}
First, we have to define success before being able to measure it. We define success as our game returning 1 and failure as our game returning 0.
This definition guarantees that the attacker succeeds if and only if he manages to remain undetected by the defender while completing his objective (usually misclassification).
Let $G(O, D', Evaluate_b, D, Train, Classify)  \rightarrow \{0,1\}$ be an instance of the game. What is currently used to measure success is the following:
\begin{definition}[Expected success rate]\label{def:exp_suc_rate}
    We define the expected success rate (ESR) as:
    \begin{gather}
    \xi_{G} = \mathbb{E}[G(O, D', Evaluate_b, D, Train, Classify)]
    \end{gather}
\end{definition}
It is also know as Attack Success Rate (ASR). However, it can be deceiving as it rewards attacking already poorly performing models and, without additional information, can misrepresent an attack's performance.
Let $G_b(\mathcal{O}_{X}, Evaluate_b, D, Train, Classify)$ be the version of our game where instead of being given adversarial examples, the defender is provided with benign samples. We use the expected success rate on this game $\Upsilon_{G_b}$
as a lower-bound for attack performance.
For simplicity's sake, we write $G(O, D', Evaluate_b, D, Train, Classify)$ as $G$ and $G_b(\mathcal{O}_{X}, Evaluate_b, D, Train, Classify)$ as $G_b$. We instead propose the following score as an alternative performance measurement.
\begin{definition}[Relative performance score]\label{def:rps}
    We define the relative performance score as:
    \begin{gather}
    \Upsilon_{G, G_b} = \xi_{G}^2 - \xi_{G_b}^2
    \end{gather}
\end{definition}

This score is contained between $-1$ and $1$. This relative score alleviates the aforementioned problem.
The score is negative when the attack performs worse than just using benign samples and is positive otherwise. 
We believe it captures the various trade-offs when mounting attacks and defenses more fairly than the expected success rate or even a shifted expected success rate by the lower-bound ($G_b$).
Figure \ref{fig:rel_perf_score} provides a visual understanding of the behavior of this score when varying both $\xi_G$ and $\xi_{G_b}$.

\begin{figure}
    \centering   
    \resizebox{\columnwidth}{!}{
    \includegraphics[width=8cm, height=5cm]{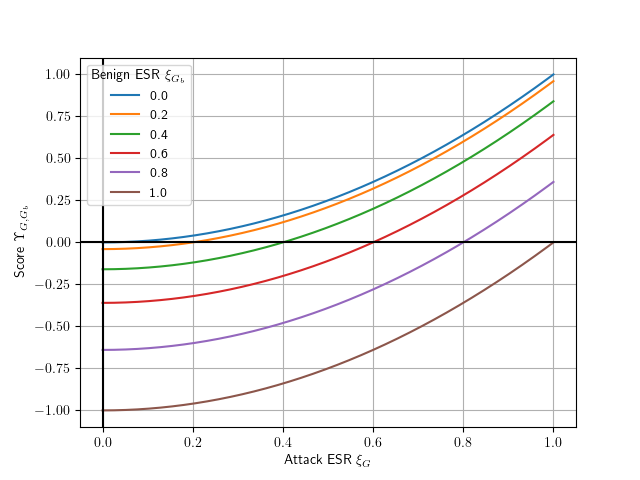}
    }
    \caption{Relative Performance Score for various benign and adversary ESRs.}\label{fig:rel_perf_score}  
\end{figure}

\subsubsection{Example: PGD}\label{subsec:showcase}
We will now showcase an application of our formalization to an existing well-known and well-performing attack: Projected Gradient Descent (PGD) \cite{madry2019deep}. This should function as a template for the practical application of our formalization and the associated game.
As a reminder, the multistep PGD attack can be defined as the following:
\begin{gather}
    x^{t+1} = \Pi_{x+\mathcal{S}}(x^{t} + \alpha \text{ sgn}(\nabla_xL(\theta,x,y))) 
\end{gather}
where $x$ is the original benign sample with its label $y$, $x^{t}$ is the input at step $t$ of the attack, $\mathcal{S} \subseteq \mathbb{R}^d $ is the set of allowed perturbations ($d$ is the input dimension), $\theta$ are the model parameters, $L$ is a loss function, sgn is the sign function (-1 if input is negative, 1 otherwise) and $\alpha$ is the step size.
In their paper, they consider the following adversaries:
\begin{enumerate}[noitemsep, topsep=0pt]
    \item "White-box attacks with PGD for a different number of iterations and restarts, denoted by source $A$".
    \item "White-box attacks with PGD using the Carlini-Wagner (CW) loss function (directly optimizing the difference between correct and incorrect logits)".
    \item "Black-box attacks from an independently trained copy of the network, denoted $A'$".
    \item "Black-box attacks from a version of the same network trained only on natural examples, denoted $A_{nat}$".
    \item "Black-box attacks from a different convolution architecture, denoted B".
\end{enumerate}
Table \ref{table:showcase} summarizes the adversary knowledge used in each version.


\begin{table}[!ht]
    \centering
    \resizebox{\columnwidth}{!}{

    \begin{tabular}{@{}cccccc@{}}
        \toprule
        \textbf{Attack} & \textbf{Model} & \textbf{Data} & \textbf{Train} & \textbf{Defense} \\
        \midrule
        1 &  Parameters & $\emptyset$ & \begin{tabular}[c]{@{}c@{}}Training \\ Information \\ (loss function)\end{tabular} & $\emptyset$ \\
        \addlinespace
        2 & Parameters & $\emptyset$ & $\emptyset$ & \begin{tabular}[c]{@{}c@{}}Full\\  Awareness\end{tabular} \\
        \addlinespace
        3 & \begin{tabular}[c]{@{}c@{}}Scores \& \\ Architecture\end{tabular} & \begin{tabular}[c]{@{}c@{}}Training \\ Data\end{tabular} & \begin{tabular}[c]{@{}c@{}}Original \\ Function\end{tabular} & \begin{tabular}[c]{@{}c@{}}Full\\  Awareness\end{tabular}\\
        \addlinespace
        4 & \begin{tabular}[c]{@{}c@{}}Scores \& \\ Architecture\end{tabular} & \begin{tabular}[c]{@{}c@{}}Training \\ Data\end{tabular} & \begin{tabular}[c]{@{}c@{}}Original \\ Function\end{tabular} & $\emptyset$ \\
        \addlinespace
        5 &  \begin{tabular}[c]{@{}c@{}}Possible \\ Architectures\end{tabular} & \begin{tabular}[c]{@{}c@{}}Training \\ Data\end{tabular} & \begin{tabular}[c]{@{}c@{}}Original \\ Function\end{tabular} & \begin{tabular}[c]{@{}c@{}}Full\\  Awareness\end{tabular} \\
        \bottomrule
    \end{tabular}
    }
    \caption{PGD attacks knowledge table}\label{table:showcase}

\end{table}

Our knowledge tables allow for a clear and concise representation of the adversary's knowledge of the defender's information. We will then describe the components of the game for the first adversary (Attack \#1). This is meant as a template to showcase how one would describe their attack using our game. To do this, we need to define the following components: $O$, $D$, $D'$, $Evaluate_b$, $Classify$, $Train$, and $AdvGen$.

As shown in Table \ref{table:showcase}, we have the associated oracles used by the attacker in Attack \#1 that are the following: $O = \{\mathcal{O}_M, \mathcal{O}_{T_1}, \mathcal{O}_{FA}, \mathcal{O}^{0, 0}_X, \mathcal{O}_{Dist}\}$. For the training information (loss function) of Attack \# 1, we construct $\mathcal{O}_{T_1}$, we let \\$T_1 = \{Train' | g^1_{Train}(Train') = True\}$ where $g^1_{Train}$ is the function that returns True when $Train'$ uses the same loss function as $Train$.

$D$ varies depending on the dataset used, for example, on MNIST \cite{mnist}, they use an \textit{indistinguishable-perturbation} distinguisher with the $l_\infty$-norm metric and a maximum allowed perturbation of $\epsilon = 0.3$ (for pixel values between 0 and 1).
However, for CIFAR-10 \cite{cifar10}, they also use an \textit{indistinguishable-perturbation} distinguisher with the $l_\infty$-norm metric, but they use a maximum allowed perturbation of $\epsilon = 8$ (for RGB pixel values between 0 and 255).
They also explore a variant of their attack in the case of an $l_2$-bounded adversary but to keep this concise we do not analyze it. $\mathcal{O}_{Disc}(x) = D' = D$ for any $x \in \{0,1\}$ as it is assumed that the attacker knows exactly the setting it is in.

Their defense is part of the training process (adversarial training), therefore $Classify(M, x) = M(x)$. The adversary's goal is to modify the attacked model's accuracy, meaning it is an untargeted attack. Hence, $b=0$ and  $Evaluate_0(y, r) = \mathbb{I}[(r \neq y)]$.
$Train$ (as defined) is the algorithm/code used to train the attacked model, due to the sheer complexity of the algorithm, we will instead point out that they link their code in their paper \cite{madry2019deep}. We can summarize it as a standard Stochastic Gradient Descent (SGD) based model training algorithm that incorporates 
adversarial training as a defense. Finally, for the adversarial example generation, they use a grounded process, hence $a=0$ and $AdvGen$ can be defined as the algorithm described in Figure \ref{alg:pgd} (they have three hyperparameters $t$, $\alpha$, $\epsilon$).
\begin{figure}[h]

\pseudocodeblock[linenumbering, space=keep]{
\textbf{ $AdvGen(O)$} \< \< t, \alpha, \epsilon \\[0.2\baselineskip ][\hline ]
\mathcal{O}_M, \mathcal{O}_{Train}, \mathcal{O}_{FA}, \mathcal{O}^{0, 0}_X, \mathcal{O}_{Dist} \gets O \\
i \gets 0 \\
x, y \gets \mathcal{O}^{0, 0}_{X}(0) \\
\theta_M \gets \mathcal{O}_M(0) \\
\cdots, L, \cdots \gets \mathcal{O}_{Train}(0) \\
x^i \gets x  \\
\textbf{While } i < t \\
\textbf{Do }\\
x^{i+1} \gets \Pi_{x + \mathcal{S}}(x^i + \alpha \text{ sgn}(\nabla_xL(\theta_M,x,y))) \\
x^{i+1} \gets \text{clip}(x^{i+1}, x - \epsilon, x + \epsilon) \\
x^{i+1} \gets \text{clip}(x^{i+1}, 0, 1) \\
i \gets i + 1 \\
\textbf{Done } \\
\textbf{Return } x^i
}
\caption{PGD Attack Algorithm}\label{alg:pgd}
\end{figure}

\section{Survey Summary \& Methodology}
In Tables~\ref{tab:attack_summary} and~\ref{tab:appendix_attack_summary} in Appendix~\ref{app:appendix_attack_summary}, we summarize the findings of our survey by reporting the information used by each attack in the surveyed papers as well as the metric used to craft the adversarial examples. In the case where papers evaluate their attacks under multiple scenarios, we split the attack into multiple variants lettered (A), (B), and so on.

We focus on recent papers, papers published in 2022 or after to narrow the field of the
search and allow for an up-to-date perspective of the field of image classification adversarial research. We gathered a total of eighty-three papers  \cite{zhang2022imagenet, Liu_Lu_Xiong_Zhang_Xiong_2023, Shamshad_Naseer_Nandakumar_2023, Zhong_2022_CVPR, Zou_Duan_Li_Zhang_Pan_Pan_2022, Wei_Zhao_2023, Wei_Chen_Wu_Jiang_2023, Weng_Luo_Lin_Li_Zhong_2023, Zhang_Tan_Sun_Zhao_Zhang_Li_2023,
Tao_Wu_2022, 10147340, Deng_Xiao_Li_He_Wang_2023, Ge_Shang_Liu_Liu_Wang_2023, Ge_Shang_Liu_Liu_Wan_Feng_Wang_2023, Liu_Ge_Zhou_Shang_Liu_Jiao_2023, Wang_Yang_Feng_Sun_Guo_Zhang_Ren_2023, Tramer_2022, Byun_Kwon_Cho_Kim_Kim_2023, Chen_Yin_Chen_Chen_Liu_2023,
 Li_Guo_Yang_Zuo_Chen_2023, Li_Guo_Zuo_Chen_2023, Qin_Fan_Liu_Shen_Zhang_Wang_Wu_2022, Yang_Lin_Zhang_Yang_Zhao_2023, Yang_Lin_Li_Shen_Zhou_Wang_Liu_2023, Qi_Huang_Panda_Henderson_Wang_Mittal_2023, Pintor_Demetrio_Sotgiu_Demontis_Carlini_Biggio_Roli_2022, Pomponi_2022, Li_decisionbased2022, liu2022practical, liu2022transferable, Apruzzese_Anderson_Dambra_Freeman_Pierazzi_Roundy_2023, 
 Agarwal_Ratha_Vatsa_Singh_2022, Aldahdooh_Hamidouche_Fezza_Deforges_2022, Altinisik_Messaoud_Sencar_Chawla_2023, Byun_2022_CVPR, chenImageRetrieval2023, Chen_Liu_2023, Crecchi_Melis_Sotgiu_Bacciu_Biggio_2022, Dai_Feng_Chen_Lu_Xia_2022, Dunmore_Jang-Jaccard_Sabrina_Kwak_2023, Dyrmishi_Ghamizi_Simonetto_Traon_Cordy_2023, Freiesleben_2022, Gong_Deng_2022, Gubri_Cordy_Papadakis_Traon_Sen_2022, Hernandez-Castro_Liu_Serban_Tsingenopoulos_Joosen_2022, Hu_Li_Yuan_Cheng_Yuan_Zhu_2022, Jun-hua_Ye-xin_Chuan-lun_Jun-yang_Xing-yu_Zhi-song_2022, Lee_Kim_2023, Li_Cheng_Hsieh_Lee_2022, Liang_Wu_Hua_Zhang_Xue_Song_Xue_Ma_Guan_2023, Liang_He_Zhao_Jia_Li_2022, Lin_Hsu_Chen_Yu_2022, Liu_Zhang_Mo_Xiang_Li_Cheng_Gao_Liu_Chen_Wei_2022, Luo_2022_CVPR, midtlidMASSA, Namiot_Ilyushin_2022, Pang_Zhang_He_Dong_Su_Chen_Zhu_Liu_2022, Pawelczyk_Agarwal_Joshi_Upadhyay_Lakkaraju_2022, Pinot_Meunier_Yger_Gouy-Pailler_Chevaleyre_Atif_2022,  qi2022, Qian_Huang_Wang_Zhang_2022, Ruiz_Kortylewski_Qiu_Xie_Bargal_Yuille_Sclaroff_2022, Sajeeda_Hossain_2022, Singh_Awasthi_Urvashi_2022, Tuna_Catak_Eskil_2022, Vos_Verwer_2022, Sheatsley_Hoak_Pauley_McDaniel_2022, Gubri_Cordy_Traon_2023, Li_Yu_Huang_2023, 10219877, wang_2022_vit, pmlr-v162-yamamura22a, long2022frequency, chen2023diffusion, chen2023contentbased, dai2023advdiff, liu2023diffprotect, lin2023diffusion, diffpure, bose2020adversarial, Wang_2021_ICCV, aigan2021, chen2022shape}, that we then narrowed to twenty using quality and relevance \cite{Byun_2022_CVPR, Gubri_Cordy_Papadakis_Traon_Sen_2022, Li_decisionbased2022, liu2022transferable, liu2022practical, Luo_2022_CVPR, midtlidMASSA, Pomponi_2022, wang_2022_vit, pmlr-v162-yamamura22a, long2022frequency, zhang2022imagenet, Zhong_2022_CVPR, Zou_Duan_Li_Zhang_Pan_Pan_2022, chen2023diffusion, chen2023contentbased, bose2020adversarial, Wang_2021_ICCV, aigan2021, chen2022shape}. None of the papers gathered used any defense information (as it is more relevant to adaptive attacks) so we exclude it from Table \ref{tab:attack_summary}. Table \ref{tab:attack_summary} contains only the papers included in the evaluation (Section \ref{sec:evaluation}). For the rest, we refer to Table \ref{tab:appendix_attack_summary} in Appendix \ref{app:appendix_attack_summary}.

\begin{table}[h!]
    \centering

    \resizebox{\columnwidth}{!}{
    \begin{tabular}{@{}cccccc@{}}
        \toprule
        \begin{tabular}[c]{@{}c@{}}\textbf{Attack}\\ (Targeted)\end{tabular} & \textbf{Model} & \textbf{Data} & \textbf{Train}& \textbf{Metric} \\
        \midrule
        \begin{tabular}[c]{@{}c@{}}LGV (A) \\\cite{Gubri_Cordy_Papadakis_Traon_Sen_2022}\end{tabular} & \begin{tabular}[c]{@{}c@{}}Possible \\ Architectures\end{tabular}& \begin{tabular}[c]{@{}c@{}}Training\\Data\end{tabular} & \begin{tabular}[c]{@{}c@{}}Training \\ Function \end{tabular} & l$_\infty$ \\
        \addlinespace
        \begin{tabular}[c]{@{}c@{}}LGV (B)\\ \cite{Gubri_Cordy_Papadakis_Traon_Sen_2022}\end{tabular} & Parameters & $\emptyset$ & $\emptyset$ & l$_\infty$\\
        \midrule
        \begin{tabular}[c]{@{}c@{}}Pixle\\ \cite{Pomponi_2022} (both)\end{tabular}& Scores & $\emptyset$ & $\emptyset$ & l$_0$\\
        \midrule
        \begin{tabular}[c]{@{}c@{}}MASSA \\ \cite{midtlidMASSA}\end{tabular} & Labels & $\emptyset$ & $\emptyset$ & l$_2$\\
        \midrule
        \begin{tabular}[c]{@{}c@{}}SSAH (A) \\ \cite{Luo_2022_CVPR} (both)\end{tabular}& Parameters & \begin{tabular}[c]{@{}c@{}}Same\\Distribution\end{tabular} & $\emptyset$ & l$_\infty$\\
        \addlinespace
        \begin{tabular}[c]{@{}c@{}}SSAH (B) \\ \cite{Luo_2022_CVPR} (both)\end{tabular}& \begin{tabular}[c]{@{}c@{}}Possible \\ Architectures\end{tabular} & \begin{tabular}[c]{@{}c@{}} Other Data\end{tabular} & \begin{tabular}[c]{@{}c@{}}Training \\ Function \end{tabular} & l$_\infty$\\
        \midrule
        \begin{tabular}[c]{@{}c@{}}BIA (A) \\ \cite{zhang2022imagenet} \end{tabular}& \begin{tabular}[c]{@{}c@{}}Possible \\ Architectures\end{tabular} & Other Data & \begin{tabular}[c]{@{}c@{}}Training \\ Function \end{tabular} & l$_\infty$\\
        \addlinespace
        \begin{tabular}[c]{@{}c@{}}BIA (B) \\ \cite{zhang2022imagenet}\end{tabular} & \begin{tabular}[c]{@{}c@{}}Possible \\ Architectures\end{tabular} & \begin{tabular}[c]{@{}c@{}}Training\\Data\end{tabular} & \begin{tabular}[c]{@{}c@{}}Training \\ Function \end{tabular} & l$_\infty$\\
        \addlinespace
        \begin{tabular}[c]{@{}c@{}}BIA (C) \\ \cite{zhang2022imagenet} \end{tabular} & Parameters & \begin{tabular}[c]{@{}c@{}}Training\\Data\end{tabular} & $\emptyset$ & l$_\infty$ \\
        \midrule
        \begin{tabular}[c]{@{}c@{}}ACG \\ \cite{pmlr-v162-yamamura22a} \end{tabular}& Parameters & $\emptyset$ & \begin{tabular}[c]{@{}c@{}}Loss \\ Function \end{tabular} & l$_\infty$\\
        \midrule
        \begin{tabular}[c]{@{}c@{}}ATA (A)\\ \cite{wang_2022_vit} \end{tabular}& \begin{tabular}[c]{@{}c@{}}Possible \\ Architectures\end{tabular} & \begin{tabular}[c]{@{}c@{}}Training\\Data\end{tabular} & \begin{tabular}[c]{@{}c@{}}Training \\ Function \end{tabular} & l$_0$ \\
        \addlinespace
        \begin{tabular}[c]{@{}c@{}}ATA (B) \\ \cite{wang_2022_vit} \end{tabular}& Parameters & $\emptyset$ & $\emptyset$ & l$_0$ \\
        \midrule
        \begin{tabular}[c]{@{}c@{}}AEG (A) \\ \cite{bose2020adversarial} \end{tabular}& Architecture & \begin{tabular}[c]{@{}c@{}}Same\\Distribution\end{tabular} & \begin{tabular}[c]{@{}c@{}} Loss \\ Function \& \\Optimizer \end{tabular} & l$_\infty$  \\
        \addlinespace
        \begin{tabular}[c]{@{}c@{}}AEG (B)\\ \cite{bose2020adversarial}\end{tabular} & \begin{tabular}[c]{@{}c@{}}Possible \\ Architectures\end{tabular} & \begin{tabular}[c]{@{}c@{}}Training\\Data\end{tabular} & \begin{tabular}[c]{@{}c@{}} Loss \\ Function \&  \\Optimizer\end{tabular} & l$_\infty$\\
        \midrule
        \begin{tabular}[c]{@{}c@{}}$\text{A}^3$ \\ \cite{liu2022practical} \end{tabular}& Parameters & \begin{tabular}[c]{@{}c@{}}Same\\Distribution\end{tabular} & \begin{tabular}[c]{@{}c@{}}Loss \\ Function \end{tabular} & l$_\infty$\\
        \bottomrule

\end{tabular}
}
\caption{Attack Summary Table. We provide the information used by the attacker for each of the categories and the metric used to craft the adverarsarial examples.}
\label{tab:attack_summary}
\end{table}

\section{Evaluation}\label{sec:evaluation}
For this work, we focus our evaluation on indistinguishable perturbation untargeted attacks that were evaluated on the ImageNet \cite{Imagenet} and CIFAR-10 \cite{cifar10} datasets, as these are the most studied datasets in the attacks we surveyed.
To compare attacks, we only compare evaluations within datasets and we only work with results presented by the authors themselves in their papers. We first, if possible, compare papers against models that are from the exact same source (public provider). If that's not possible, we compare models of the same architecture while also looking at their benign accuracies. In the case where a paper does not provide benign accuracies, we use the benign accuracy from another paper but report the result in italic to show its unreliability. 

\subsection{Attacks on the ImageNet dataset}\label{subseq:ind_un_imagenet}

First, we tackle the most evaluated dataset that we study in our work, with eleven variants of attacks evaluated against it: ImageNet \cite{Imagenet}. We perform four comparisons in total.
We summarize the results of our comparisons on ImageNet in Table \ref{tab:imagenet}. The only paper to evaluate their attack on defended models is ACG \cite{pmlr-v162-yamamura22a} where they get an average ESR of 68.50\% and an average score of 0.314.

\begin{table}[]
\resizebox{\columnwidth}{!}{
\begin{tabular}{cccc|cc}
\begin{tabular}[c]{@{}c@{}}Model \\ Information\end{tabular}                      & \begin{tabular}[c]{@{}c@{}}Data \\ Information\end{tabular}                  & \begin{tabular}[c]{@{}c@{}}Training \\ Information\end{tabular}               & \begin{tabular}[c]{@{}c@{}}Attack\\ Variant\end{tabular} & \begin{tabular}[c]{@{}c@{}}Und.\\ ESR\end{tabular} & \begin{tabular}[c]{@{}c@{}}Und.\\ Score\end{tabular} \\ \hline
\multirow{4}{*}{\begin{tabular}[c]{@{}c@{}}Possible\\ Architectures\end{tabular}} & \multirow{3}{*}{\begin{tabular}[c]{@{}c@{}}Training\\Data\end{tabular}}                                               & \multirow{3}{*}{\begin{tabular}[c]{@{}c@{}}Training \\ Function\end{tabular}} & LGV (A)                                                  & \textit{72.4}                                            & \textit{0.497}                                             \\
                                                                                  &                                                                              &                                                                               & BIA (B)                                                  & 93.48                                                    & 0.811                                                      \\
                                                                                  &                                                                              &                                                                               & ATA (A)                                                  & \textit{79.65}                                           & \textit{0.524}                                             \\ \cline{2-6} 
                                                                                  & \begin{tabular}[c]{@{}c@{}}Training \\Data \&\\ Same \\Distribution\end{tabular} & \begin{tabular}[c]{@{}c@{}}Training\\ Function\end{tabular}                   & SSAH (B)                                                 & \textit{19.14}                                           & N/A                                                            \\ \hline
Labels                                                                            & $\emptyset$                                                                       & $\emptyset$                                                                        & MASSA                                                    & \textit{99.49}                                           & \textit{0.916}                                             \\ \hline
Scores                                                                            & $\emptyset$                                                                       & $\emptyset$                                                                        & Pixle                                                    & \textit{98.5}                                            & \textit{0.897}                                             \\ \hline
\multirow{5}{*}{Parameters}                                                       & \begin{tabular}[c]{@{}c@{}}Training\\Data\end{tabular}                                                                & $\emptyset$                                                                        & BIA (C)                                                  & 97.82                                                    & 0.883                                                      \\ \cline{2-6} 
                                                                                  & \begin{tabular}[c]{@{}c@{}}Same\\Distribution\end{tabular}                                                           & $\emptyset$                                                                        & SSAH (A)                                                 & 95.56                                                    & 0.915                                                      \\ \cline{2-6} 
                                                                                  & \multirow{3}{*}{$\emptyset$}                                                       & \multirow{2}{*}{$\emptyset$}                                                        & LGV (B)                                                  & \textit{97.1}                                            & \textit{0.885}                                             \\
                                                                                  &                                                                              &                                                                               & ATA (B)                                                  & \textit{100.0}                                           &    N/A                                                        \\ \cline{3-6} 
                                                                                  &                                                                              & \begin{tabular}[c]{@{}c@{}}Loss \\ Function\end{tabular}                      & ACG                                                      &    N/A                                                      &           N/A                                                 \\ \hline
\end{tabular}
}
\caption{ImageNet information and attack results. "Und." stands for undefended.}
\label{tab:imagenet}
\end{table}

\noindent
The takeaways for ImageNet are the following:
\begin{itemize}[noitemsep, topsep=0pt]
    \item Only query-based access to the model and no other information is sufficient to reduce an undefended model's accuracy to almost 0\%.
    \item Transferable settings can achieve extremely high ESRs (93.48\% for BIA (B)), which shows that, at least in the undefended case, transferable attacks that use additional data and training information (usually to train surrogates) can 
    achieve ESRs close to query-based and white-box attacks. Unfortunately, since none of the transferable attacks we studied evaluated defended models, we cannot say if this extends to defended models.
    \item ImageNet-defended models are mostly broken by white-box attacks. ACG obtains an average ESR of 68.5\% and an average score 0.314. While, defended models do perform better than undefended ones, the accuracy loss is still significant.
\end{itemize}

\subsubsection{LGV (A) \& BIA (B) \& ATA (A) \& SSAH (B)}\label{subsubsec:1}
We start by comparing the different transferable attacks: LGV (A) \cite{Gubri_Cordy_Papadakis_Traon_Sen_2022}, BIA (B) \cite{zhang2022imagenet}, ATA (A) \cite{wang_2022_vit}, SSAH (B) \cite{Luo_2022_CVPR}.
All these attacks attack undefended papers. Some share architectures, however, not all of them provide benign accuracies for the models they evaluate.
We compile the results in Tables \ref{table:ind_un_imagenet_blue_comparison_asr}.
BIA (B), LGV(A), and SSAH (B) both use the $l_\infty$-norm but BIA (B) uses a perturbation budget of 10/256 whereas SSAH (B) uses one of 8/256. LGV (A) does not specify the budget used. ATA (A) uses the $l_0$-norm, where they are only allowed
to fully modify 1024 pixels. Unfortunately, this means that the following comparison is disparate.  
\begin{table}[]
    \centering
    \resizebox{\columnwidth}{!}{
    \begin{tabular}{@{}cccccc@{}}
        \toprule
        \textbf{Model}  & \begin{tabular}[c]{@{}c@{}}\textbf{Benign}\\ \textbf{ESR} \\ \textbf{ (BIA)}\end{tabular} & \begin{tabular}[c]{@{}c@{}}\textbf{LGV (A)}\\ \textbf{ESR /}\\ \textbf{Score}\end{tabular} & \begin{tabular}[c]{@{}c@{}}\textbf{BIA (B)}\\ \textbf{ESR /}\\ \textbf{Score}\end{tabular} & \begin{tabular}[c]{@{}c@{}}\textbf{ATA (A)}\\ \textbf{ESR /}\\ \textbf{Score}\end{tabular} & \begin{tabular}[c]{@{}c@{}}\textbf{SSAH (B)}\\ \textbf{ESR /}\\ \textbf{Score}\end{tabular} \\
        \midrule
        \begin{tabular}[c]{@{}c@{}}ResNet-50\\ \cite{resnet}\end{tabular} & 24.39 & \begin{tabular}[c]{@{}c@{}}N/A \\N/A\end{tabular} & \begin{tabular}[c]{@{}c@{}}\textbf{95.56} \\ \textbf{0.854} \end{tabular}& \begin{tabular}[c]{@{}c@{}} \textit{87.31} \\ \textit{0.703} \end{tabular} & \begin{tabular}[c]{@{}c@{}}N/A \\ N/A \end{tabular}\\
        \midrule
        \begin{tabular}[c]{@{}c@{}}VGG-16 \\ \cite{vgg} \end{tabular} & N/A & \begin{tabular}[c]{@{}c@{}}N/A \\ N/A \end{tabular}& \begin{tabular}[c]{@{}c@{}}N/A \\ N/A \end{tabular} & \begin{tabular}[c]{@{}c@{}}\textbf{87.46} \\ N/A \end{tabular} & \begin{tabular}[c]{@{}c@{}}\textit{19.14} \\ N/A \end{tabular} \\
        \midrule
        \begin{tabular}[c]{@{}c@{}}DenseNet-121 \\ \cite{densenet} \end{tabular}& 25.78 & \begin{tabular}[c]{@{}c@{}}N/A \\ N/A \end{tabular}& \begin{tabular}[c]{@{}c@{}}\textbf{96.02} \\ \textbf{0.856}\end{tabular} & \begin{tabular}[c]{@{}c@{}}\textit{64.17} \\ \textit{0.345} \end{tabular} & \begin{tabular}[c]{@{}c@{}}N/A \\ N/A \end{tabular}\\
        \midrule
        \begin{tabular}[c]{@{}c@{}}ResNet-152 \\ \cite{resnet}\end{tabular} & 22.66 & \begin{tabular}[c]{@{}c@{}}\textit{89.6} \\ \textit{0.751}\end{tabular}& \begin{tabular}[c]{@{}c@{}}\textbf{94.15} \\ \textbf{0.835}\end{tabular} & \begin{tabular}[c]{@{}c@{}}N/A \\ N/A \end{tabular}& \begin{tabular}[c]{@{}c@{}}N/A \\ N/A \end{tabular} \\
        \midrule
        \begin{tabular}[c]{@{}c@{}}VGG-19 \\ \cite{vgg} \end{tabular} & 29.05 & \begin{tabular}[c]{@{}c@{}}\textit{82.2} \\ \textit{0.591} \end{tabular}& \begin{tabular}[c]{@{}c@{}}\textbf{95.91} \\ \textbf{0.835}\end{tabular} & \begin{tabular}[c]{@{}c@{}}N/A \\ N/A \end{tabular}& \begin{tabular}[c]{@{}c@{}}N/A \\ N/A \end{tabular}\\
        \midrule
        \begin{tabular}[c]{@{}c@{}}Inception-v3 \\ \cite{inceptionv3}\end{tabular} & 23.81 & \begin{tabular}[c]{@{}c@{}}\textit{45.4} \\ \textit{0.149}\end{tabular} & \begin{tabular}[c]{@{}c@{}}\textbf{85.76} \\ \textbf{0.679}\end{tabular} & \begin{tabular}[c]{@{}c@{}}N/A \\ N/A \end{tabular}& \begin{tabular}[c]{@{}c@{}}N/A \\ N/A \end{tabular}\\
        \midrule
        Average & 25.14 & \begin{tabular}[c]{@{}c@{}}\textit{72.4} \\ \textit{0.497}\end{tabular} & \begin{tabular}[c]{@{}c@{}}\textbf{93.48} \\ \textbf{0.811}\end{tabular} & \begin{tabular}[c]{@{}c@{}}\textit{79.65} \\ \textit{0.524}\end{tabular} & \begin{tabular}[c]{@{}c@{}}\textit{19.14} \\ N/A\end{tabular} \\
        \midrule
        \begin{tabular}[c]{@{}c@{}}Standard \\deviation\end{tabular} & 2.46 &  \begin{tabular}[c]{@{}c@{}}\textit{23.67} \\ \textit{0.312 }\end{tabular} & \begin{tabular}[c]{@{}c@{}}4.38 \\ 0.075 \end{tabular}&  \begin{tabular}[c]{@{}c@{}}\textit{13.4} \\ \textit{0.253}\end{tabular} & \begin{tabular}[c]{@{}c@{}}N/A \\ N/A \end{tabular} \\
        \bottomrule
    \end{tabular}
    }
    \caption{ImageNet dataset indistinguishable Section \ref{subsubsec:1} undefended results.}
    \label{table:ind_un_imagenet_blue_comparison_asr}
\end{table}

BIA (B) outperforms the other attacks. Due to the lack of results for SSAH (B) and the fact that the few results are worse than the other attacks even though SSAH (B) has more information, we ignore it in what follows. All the other three attacks use the same information, there are three possible reasons for this disparity in results:
\begin{enumerate}[noitemsep, topsep=0pt]
    \item BIA (B) is a better algorithm for extracting information and building a potent attack from it.
    \item LGV (A) and ATA (A) attacked undefended models that were somehow inherently more robust (unlikely).
    \item The disparity of results is caused by the different perturbation budgets/$l_p$-norms used. 
\end{enumerate}
We suspect a combination of 1. and 3. to be the reason explaining the discrepancy in the results.
Additionally, ATA's attack is originally meant to attack vision transformer models. Hence, it could partially explain a slightly worse performance in an equitable evaluation setting.

\subsubsection{LGV (B) \& ATA (B) \& ACG} \label{subsubsec:2}
This comparison is difficult to perform, as LGV (B) and ATA (B) both attack different undefended models while not providing any benign accuracies. We can use the benign accuracy of the BIA paper as a substitute for LGV (B) but not for ATA (B).
On the other hand, ACG \cite{pmlr-v162-yamamura22a} attacks only defended models, therefore we cannot directly compare with the other attacks.
We compile the results in Tables \ref{table:ind_un_imagenet_lgv_b_undefended} and \ref{table:ind_un_imagenet_acg_defended}.

\begin{table}[h!]
    \centering

    \resizebox{\columnwidth}{!}{
    \begin{tabular}{@{}ccccc@{}}
        \toprule
        \textbf{Model}  & \begin{tabular}[c]{@{}c@{}}\textbf{Benign}\\ \textbf{ESR}\end{tabular} & \begin{tabular}[c]{@{}c@{}}\textbf{LGV (B)}\\ \textbf{ESR / Score}\end{tabular} & \begin{tabular}[c]{@{}c@{}}\textbf{ATA (B)}\\ \textbf{ESR / Score}\end{tabular} \\
        \midrule
        ResNet-50 \cite{resnet} & 24.39 (BIA) & \textit{97.1} / \textit{0.883}& N/A / N/A \\
        \addlinespace
        DeiT-T \cite{deit} & N/A & N/A / N/A& \textit{100.0} / N/A\\
        \addlinespace
        DeiT-S \cite{deit} & N/A & N/A / N/A& \textit{100.0} / N/A\\
        \addlinespace
        DeiT-B \cite{deit} & N/A & N/A / N/A& \textit{100.0} / N/A\\
        \bottomrule
    \end{tabular}
    }
    \caption{ImageNet dataset indistinguishable undefended results.}
    \label{table:ind_un_imagenet_lgv_b_undefended}
\end{table}

\begin{table}[h!]
    \centering

    \begin{tabular}{@{}ccccc@{}}
        \toprule
        \textbf{Model}  & \begin{tabular}[c]{@{}c@{}}\textbf{Benign}\\ \textbf{ESR}\end{tabular} & \begin{tabular}[c]{@{}c@{}}\textbf{ACG}\\ \textbf{ESR}\end{tabular} & \begin{tabular}[c]{@{}c@{}}\textbf{ACG}\\ \textbf{Score}\end{tabular} \\
        \midrule
        ResNet-50$_\text{robustness}$ \cite{robustness} & 37.44 & 69.58 & 0.344 \\
        \addlinespace
        ResNet-50$_\text{salman}$ \cite{salman2020adversarially} & 35.98 & 64.7 & 0.289 \\
        \addlinespace
        ResNet-18$_\text{salman}$ \cite{salman2020adversarially} & 47.08 & 74.34 & 0.331 \\
        \addlinespace
        WideResNet-50-2$_\text{salman}$ \cite{salman2020adversarially} & 31.54 & 60.9 & 0.271 \\
        \addlinespace
        ResNet-50$_\text{FAST\_AT}$ \cite{wong2020fast} & 44.38 & 73.0 & 0.336 \\
        \midrule
        Average & 39.28 & 68.50 & 0.314 \\
        Standard deviation & 6.34 & 5.65 & 0.032 \\
        \bottomrule
    \end{tabular}
    \caption{ImageNet dataset indistinguishable ACG defended results.}
    \label{table:ind_un_imagenet_acg_defended}
\end{table}

\subsubsection{BIA (C) \& SSAH (A) \& Section \ref{subsubsec:2} }\label{subsubsec:3}
Both BIA (C) and SSAH (A) evaluate only against undefended models, therefore comparing against ACG is not possible. We can, however, compare SSAH (A) and LGV (B) as they both evaluate against a ResNet-50 \cite{resnet}.
Since SSAH (A) publishes their ResNet-50's benign accuracy and LGV (B) does not, we also use BIA's benign accuracy from the previous evaluation. We compute LGV (B)'s score using the mean of both benign accuracies to get closer to LGV (B)'s ResNet-50's
expected benign accuracy and get a more reliable result. BIA (C) only evaluates their white-box scenario against a VGG-16 and a DenseNet-169 in the undefended setting. We compile the results in Table \ref{table:ind_un_imagenet_lgv_b_and_ssah_a_undefended}.

\begin{table}[ht]
    \centering
    \resizebox{\columnwidth}{!}{
    \begin{tabular}{@{}cccccc@{}}
        \toprule
        \textbf{Model} & \begin{tabular}[c]{@{}c@{}}\textbf{Benign}\\ \textbf{ESR}\end{tabular} & \begin{tabular}[c]{@{}c@{}}\textbf{LGV (B)}\\ \textbf{ESR} /\\ \textbf{Score}\end{tabular} & \begin{tabular}[c]{@{}c@{}}\textbf{SSAH (A)}\\ \textbf{ESR} /\\ \textbf{Score}\end{tabular} & \begin{tabular}[c]{@{}c@{}}\textbf{BIA (C)}\\ \textbf{ESR} /\\ \textbf{Score}\end{tabular}\\
        \midrule
        \begin{tabular}[c]{@{}c@{}}ResNet-50 \\\cite{resnet}\end{tabular} & \begin{tabular}[c]{@{}c@{}} 24.39 (BIA) \\ 23.85 (SSAH)\end{tabular} & \begin{tabular}[c]{@{}c@{}}\textit{97.1} \\ \textit{0.885}\end{tabular} & \begin{tabular}[c]{@{}c@{}}\textbf{98.56} \\ \textbf{0.915}\end{tabular}& \begin{tabular}[c]{@{}c@{}}NA \\ NA \end{tabular}   \\
        \midrule
        \begin{tabular}[c]{@{}c@{}}VGG-16\\ \cite{vgg}\end{tabular} & 29.86 (BIA) &  \begin{tabular}[c]{@{}c@{}}NA \\ NA \end{tabular} & \begin{tabular}[c]{@{}c@{}}NA \\ NA \end{tabular} & \begin{tabular}[c]{@{}c@{}}98.96 \\ 0.890\end{tabular} \\
        \midrule
        \begin{tabular}[c]{@{}c@{}}DenseNet-169\\ \cite{densenet}\end{tabular} & 24.25 (BIA) &  \begin{tabular}[c]{@{}c@{}}NA \\ NA \end{tabular} & \begin{tabular}[c]{@{}c@{}}NA \\ NA \end{tabular} & \begin{tabular}[c]{@{}c@{}}96.68 \\ 0.876\end{tabular} \\
        \midrule
        Average & 26.08 & \begin{tabular}[c]{@{}c@{}}\textit{97.1} \\ \textit{0.885}\end{tabular} & \begin{tabular}[c]{@{}c@{}}\textbf{98.56} \\ \textbf{0.915}\end{tabular} & \begin{tabular}[c]{@{}c@{}}97.82 \\ 0.883\end{tabular} \\
        \midrule
        \begin{tabular}[c]{@{}c@{}}Standard \\deviation\end{tabular} & 3.28 & \begin{tabular}[c]{@{}c@{}}NA \\ NA \end{tabular} &  \begin{tabular}[c]{@{}c@{}}NA \\ NA \end{tabular} & \begin{tabular}[c]{@{}c@{}}1.61 \\ 0.010\end{tabular} \\
         \bottomrule
    \end{tabular}
    }
        \caption{ImageNet dataset indistinguishable Section \ref{subsubsec:3} undefended results.}
    \label{table:ind_un_imagenet_lgv_b_and_ssah_a_undefended}
\end{table}

\subsubsection{MASSA \& Pixle}\label{subsubsec:4}
We include the results from MASSA \cite{midtlidMASSA} and Pixle \cite{Pomponi_2022} to complete our comparison.
Since neither paper provides clean accuracies, we use the average of the benign accuracies from BIA and SSAH to compute the score. This yields the compiled results
in Table \ref{table:ind_un_imagenet_massa_and_pixle_undefended}. While MASSA evaluates its performance against defended models, they are the sole paper to use the $l_2$-norm for its distinguisher. Additionally, they do not provide clean accuracies for any of their models.
Pixle on the other hand does not attack any defended models. Our results show that query-based access to a model is sufficient to render it useless if the model is undefended.

\begin{table}[ht]
    \centering

    \resizebox{\columnwidth}{!}{
    \begin{tabular}{@{}ccccc@{}}
        \toprule
        \textbf{Model} & \begin{tabular}[c]{@{}c@{}}\textbf{Benign}\\ \textbf{ESR}\end{tabular} & \begin{tabular}[c]{@{}c@{}}\textbf{MASSA}\\ \textbf{ESR /Score}\end{tabular} & \begin{tabular}[c]{@{}c@{}}\textbf{Pixle}\\ \textbf{ESR / Score}\end{tabular} \\
        \midrule
        ResNet-50 \cite{resnet} & \begin{tabular}[c]{@{}c@{}} 24.39 (BIA) \\ 23.85 (SSAH)\end{tabular} & \textit{\textbf{99.4}} / \textit{\textbf{0.930}} & \textit{98.0} / \textit{0.902} \\
        \addlinespace
        VGG-16 \cite{vgg} & 29.86 (BIA) & \textit{\textbf{99.58}} / \textit{\textbf{0.902}} & \textit{99.0} / \textit{0.891} \\
        \midrule
        Average & 26.99 & \textbf{\textit{99.49}} / \textit{\textbf{0.916}}& \textit{98.5} / \textit{0.897} \\
        Standard deviation & 4.06 & 0.12 / 0.019 & 0.7 / 0.008\\
        \bottomrule
    \end{tabular}
    }
    \caption{ImageNet dataset indistinguishable Section \ref{subsubsec:4} undefended results.}
    \label{table:ind_un_imagenet_massa_and_pixle_undefended}
\end{table}

\subsection{Attacks on the CIFAR10 dataset}\label{subseq:ind_un_cifar10}

The second dataset we study is the CIFAR-10 dataset \cite{cifar10}.
We can perform four comparisons on a total of seven variants of attacks. 
We summarize our results in Table \ref{tab:cifar10}.

\begin{table}[]
\resizebox{\columnwidth}{!}{
\begin{tabular}{cccc|cc}
\begin{tabular}[c]{@{}c@{}}Model \\ Information\end{tabular}                      & \begin{tabular}[c]{@{}c@{}}Data \\ Information\end{tabular} & \begin{tabular}[c]{@{}c@{}}Training \\ Information\end{tabular} & \begin{tabular}[c]{@{}c@{}}Attack\\ Variant\end{tabular} & \begin{tabular}[c]{@{}c@{}}Und.\\ ESR /\\ Score\end{tabular} & \begin{tabular}[c]{@{}c@{}}Def.\\ ESR /\\ Score\end{tabular} \\ \hline
\multirow{2}{*}{\begin{tabular}[c]{@{}c@{}}Possible\\ Architectures\end{tabular}} & \begin{tabular}[c]{@{}c@{}}Training\\ Data\end{tabular}                                               & \begin{tabular}[c]{@{}c@{}}Loss \&\\ Optimizer\end{tabular}     & \begin{tabular}[c]{@{}c@{}}AEG\\(B)\end{tabular}                                               & \begin{tabular}[c]{@{}c@{}}92.62 \\ 0.848 \end{tabular}                                                   & \begin{tabular}[c]{@{}c@{}}47.52 \\ 0.162 \end{tabular}                                                 \\ \cline{2-6} 
                                                                                  & Other Data                                                  & \begin{tabular}[c]{@{}c@{}}Training\\ Function\end{tabular}     & \begin{tabular}[c]{@{}c@{}}BIA\\(A)\end{tabular}                                               & \begin{tabular}[c]{@{}c@{}}47.19 \\ 0.219\end{tabular}                                                    & \begin{tabular}[c]{@{}c@{}}NA \\ NA \end{tabular}                                                      \\ \hline
Architecture                                                                      & \begin{tabular}[c]{@{}c@{}}Same \\Distribution\end{tabular}                                           & \begin{tabular}[c]{@{}c@{}}Loss \&\\ Optimizer\end{tabular}     & \begin{tabular}[c]{@{}c@{}}AEG\\(A)\end{tabular}                                                  & \begin{tabular}[c]{@{}c@{}}\textit{87.0} \\ \textit{N/A}\end{tabular}                                              & \begin{tabular}[c]{@{}c@{}}\textit{NA} \\ \textit{NA} \end{tabular}                                             \\ \hline
Scores                                                                            & $\emptyset$                                                       & $\emptyset$                                                           & Pixle                                                    & \begin{tabular}[c]{@{}c@{}}\textit{100.0} \\ \textit{0.983} \end{tabular}                                                     & \begin{tabular}[c]{@{}c@{}}NA \\ NA \end{tabular}                                                      \\ \hline
\multirow{3}{*}{Parameters}                                                       & \multirow{2}{*}{\begin{tabular}[c]{@{}c@{}}Same \\Distribution\end{tabular}}                          & $\emptyset$                                                           & \begin{tabular}[c]{@{}c@{}}SSAH\\(A)\end{tabular}                                                 & \begin{tabular}[c]{@{}c@{}}99.96 \\ 0.994  \end{tabular}                                                  & \begin{tabular}[c]{@{}c@{}}21.32 \\ 0.023 \end{tabular}                                                 \\ \cline{3-6} 
                                                                                  &                                                             & \begin{tabular}[c]{@{}c@{}}Loss\\ Function\end{tabular}         & A$^3$                                     & \begin{tabular}[c]{@{}c@{}}NA \\ NA \end{tabular}                                                        & \begin{tabular}[c]{@{}c@{}}42.13 \\ 0.163   \end{tabular}                                               \\ \cline{2-6} 
                                                                                  & $\emptyset$                                                       & \begin{tabular}[c]{@{}c@{}}Loss \\ Function\end{tabular}        & ACG                                                      & \begin{tabular}[c]{@{}c@{}}NA \\ NA \end{tabular}                                                        & \begin{tabular}[c]{@{}c@{}}41.70 \\ 0.160 \end{tabular}                                                  \\ \cline{1-6} 
\end{tabular}
}
\caption{CIFAR-10 information and attack results. "Und." stands for undefended and "Def." stands for defended.}
\label{tab:cifar10}
\end{table}

Unsurprisingly, we find that it requires little model information to get effective attacks if the attacker has access to additional information like data and training information.
Additionally, we still observe a strong transferable performance against defended models by AEG (B) where they obtain a score (0.167) very similar to white-box attacks (0.163 and 0.160).  
While this is still much worse than against undefended models, it would still in practice severely decrease the usability of the target model.

\subsubsection{BIA (A) \& AEG (B)}\label{subsubsec:5}

BIA (A) \cite{zhang2022imagenet} only attack a custom model, and therefore we treat their results as unreliable. Their model achieves 6.22\% benign ESR, while their attack achieves an ESR of 47.19\%, yielding a score of 0.219.
AEG (B) \cite{bose2020adversarial} on the other hand attack many common model architectures.

Since BIA (A) attacks a custom model, we have to perform an aggregate analysis of AEG (B)'s performance on both undefended and defended models to be able to compare.
We report the performance on undefended / defended models in Table \ref{table:ind_un_cifar10_undefended_aeg_b}.

\begin{table}[ht]
    \centering

    \begin{tabular}{@{}cccc@{}}
        \toprule
        \textbf{Model}  & \begin{tabular}[c]{@{}c@{}}\textbf{Benign}\\ \textbf{ESR}\end{tabular} & \begin{tabular}[c]{@{}c@{}}\textbf{AEG (B)}\\ \textbf{ESR}\end{tabular} & \begin{tabular}[c]{@{}c@{}}\textbf{AEG (B)}\\ \textbf{Score}\end{tabular} \\
        \midrule
        VGG-16 \cite{vgg} & 11.2 & 93.8 & 0.867 \\
        \addlinespace
        ResNet-18 \cite{resnet} & 13.1 & 97.3 & 0.930 \\
        \addlinespace
        WideResNet \cite{wide_resnet} & 6.8 & 85.2 & 0.721 \\
        \addlinespace
        DenseNet-121 \cite{densenet} & 11.2 & 94.1 & 0.873 \\
        \addlinespace
        Inception-v3 \cite{inceptionv3} & 9.9 & 92.7 & 0.850 \\
        \midrule 
        Average & 10.44 & 92.62 & 0.848 \\
        \addlinespace
        Standard deviation & 2.33 & 4.49 & 0.077 \\
        \midrule
        ResNet-18$_\text{ens3}$ \cite{tramer2020ensemble} & 16.8 & 52.2 & 0.244 \\
        \addlinespace
        WideResNet$_\text{ens3}$ \cite{tramer2020ensemble} & 12.8 & 49.9 & 0.232 \\
        \addlinespace
        DenseNet-121$_\text{ens3}$ \cite{tramer2020ensemble} & 21.5 & 41.4 & 0.125 \\
        \addlinespace
        Inception-v3$_\text{ens3}$ \cite{tramer2020ensemble} & 14.8 & 47.5 & 0.204 \\
        \addlinespace
        Madry-Adv \cite{madry2019deep} & 12.9 & 21.6 & 0.030 \\
        \midrule 
        Average & 15.76 & 47.52 & 0.167 \\
        \addlinespace
        Standard deviation & 3.60 & 12.37 & 0.090 \\
        \bottomrule
    \end{tabular}
    \caption{CIFAR10 dataset indistinguishable AEG (B) results.}
    \label{table:ind_un_cifar10_undefended_aeg_b}
    
\end{table}

If we treat BIA (A)'s model as undefended, then AEG (B) significantly outperforms BIA (A). This could either be due to the difference in the information required by both attacks
or that BIA (A) does not extract attack performance as well out of the information it uses as AEG (B). 

Unsurprisingly, attacking well-defended models in the transferable setting is quite a difficult task. The results in Table \ref{table:ind_un_cifar10_undefended_aeg_b} strongly imply that BIA (A)'s model is indeed 
undefended as it achieves a benign ESR much closer to the undefended models of Table \ref{table:ind_un_cifar10_undefended_aeg_b}.

\subsubsection{A$^3$ \& SSAH (A)}\label{subsubsec:6}
We can directly compare both attacks as they both attack WideResNet-34-10$_\text{TRADES}$ \cite{zhang2019defense}. We can confidently confirm that these are likely to be the same models as both papers report the same benign accuracy.
We present the results in Table \ref{table:ind_un_cifar10_purple_comparison}.

\begin{table}[ht]
    \centering

    \begin{tabular}{@{}ccccc@{}}
        \toprule
        \textbf{Model}  & \begin{tabular}[c]{@{}c@{}}\textbf{Benign}\\ \textbf{ESR}\end{tabular} & \begin{tabular}[c]{@{}c@{}}\textbf{SSAH (A)}\\ \textbf{ESR }/ \\ \textbf{Score}\end{tabular} & \begin{tabular}[c]{@{}c@{}}\textbf{A$^3$}\\ \textbf{ESR }/ \\ \textbf{Score}\end{tabular}\\
        \midrule
        WRN-34-10$_\text{TRADES}$ \cite{zhang2019defense} & 15.08 & \begin{tabular}[c]{@{}c@{}}21.32 \\ 0.023 \end{tabular}& \begin{tabular}[c]{@{}c@{}}\textbf{46.99} \\ \textbf{0.198} \end{tabular}\\
        \bottomrule
    \end{tabular}
    \caption{CIFAR10 dataset indistinguishable SSAH (A) and A$^3$ defended results. "WRN" stands for WideResNet.}
    \label{table:ind_un_cifar10_purple_comparison}
\end{table}

We can see that although, as stated by the authors, A$^3$ is an attack that focuses both on ESR and runtime, it still significantly outperforms SSAH (A) when attacking a well-defended model.
It is doubtful that having access to the training loss function is a strong enough difference between the two attacks to justify the performance gap. It is more likely that SSAH (A) is inefficient 
at extracting the information it has access to to mount a potent attack against defended models.

\subsubsection{ACG \&  A$^3$}

We again compare ACG \cite{pmlr-v162-yamamura22a} and A$^3$ on defended models. On CIFAR10, ACG and A$^3$ attack 11 models in common from RobustBench \cite{robustbench}. The results are summarized in Table \ref{table:ind_un_cifar10_blue_comparison}\footnote{* WideResNet-28-10$_\text{pretraining}$ misreported as WideResNet-34-10$_\text{pretraining}$ in the A$^3$ paper.}.

\begin{table}[ht]
    \centering
    \resizebox{\columnwidth}{!}{

    \begin{tabular}{@{}ccccccc@{}}
        \toprule
        \textbf{Model}  & \begin{tabular}[c]{@{}c@{}}\textbf{Benign}\\ \textbf{ESR}\end{tabular} & \begin{tabular}[c]{@{}c@{}}\textbf{A$^3$}\\ \textbf{ESR / Score}\end{tabular} & \begin{tabular}[c]{@{}c@{}}\textbf{ACG}\\ \textbf{ESR / Score}\end{tabular}\\
        \midrule
        \begin{tabular}[c]{@{}c@{}}WRN-34-10$_\text{TRADES}$ \\ \cite{zhang2019defense}\end{tabular} & 15.08 & 46.99 / 0.198 & \textbf{47.18} / \textbf{0.200} \\
        \midrule
        \begin{tabular}[c]{@{}c@{}}WRN-34-20$_\text{LBGAT}$ \\ \cite{lbgat} \end{tabular}& 11.3 & \textbf{46.54} / \textbf{0.204}& 46.23 / 0.201 \\
        \addlinespace
        \begin{tabular}[c]{@{}c@{}}WRN-34-10$_\text{LBGAT}$ \\ \cite{lbgat} \end{tabular}& 11.78 & \textbf{47.24} / \textbf{0.209} & 46.9 / 0.206 \\
        \midrule
        \begin{tabular}[c]{@{}c@{}}WRN-34-20$_\text{Overfitting}$ \\ \cite{overfitting}\end{tabular} & 14.66 & \textbf{46.67} / \textbf{0.196} & 45.69 / 0.187 \\
        \midrule
        WRN-70-16$_\text{Fixing}$ \cite{fixingdata} & 11.46 & \textbf{35.81} / \textbf{0.115} & 35.27 / 0.111 \\
        \addlinespace
        WRN-28-10$_\text{Fixing}$ \cite{fixingdata} & 12.67 & \textbf{39.34} /\textbf{0.139} & 38.8 / 0.134 \\
        \midrule
        WRN-28-10$_\text{AWP}$ \cite{awp} & 11.75 & \textbf{40.02} / \textbf{0.146} & 39.7 / 0.144 \\
        \midrule
        WRN-70-16$_\text{ULAT}$ \cite{ulat} & 8.9 & \textbf{34.22} / \textbf{0.109}& 33.7 / 0.106 \\
        \addlinespace
        WRN-70-16$_\text{ULAT}$ \cite{ulat} & 14.71 & \textbf{42.92} / \textbf{0.163} & 42.45 / 0.159 \\
        \addlinespace
        WRN-34-20$_\text{ULAT}$ \cite{ulat} & 14.36 & \textbf{43.24} / \textbf{0.166} & 42.86 / 0.163 \\
        \addlinespace
        WRN-28-10$_\text{ULAT}$ \cite{ulat} & 10.52 & \textbf{37.3} / \textbf{0.128}& 36.9 / 0.125 \\
        \midrule
        \begin{tabular}[c]{@{}c@{}}WRN-28-10$_\text{pretraining}$* \\ \cite{pretraining}\end{tabular} & 12.89 & \textbf{45.24} / \textbf{0.188} & 44.75 / 0.184 \\
        \midrule
        Average & 12.51 & \textbf{42.13} / \textbf{0.163} & 41.70 / 0.160 \\
        \addlinespace
        Standard deviation & 1.81 & 4.67 / 0.036 & 4.72 / 0.036 \\
        \bottomrule
    \end{tabular}
    }
        \caption[CIFAR10 dataset indistinguishable A$^3$ and ACG defended comparison.]{CIFAR10 dataset indistinguishable A$^3$ and ACG defended results. "WRN" stands for WideResNet.
}
    \label{table:ind_un_cifar10_blue_comparison}
\end{table}

ACG outperforms both A$^3$ and SSAH (A) on WideResNet-34-10$_\text{TRADES}$ but otherwise loses to A$^3$ on the other models and overall. We observe a small, discrepancy between ACG and A$^3$ in terms of 
performance. It could either be due to the information used or the inherent algorithmic differences.
Finally, we are left with comparing all seven attacks against one another for our final comparison.

\subsubsection{AEG (A) \& Pixle \& Rest}
Unfortunately, due to the amount of papers to compare, we reach the limits of what we can do while keeping the comparisons fair. We go through the few comparisons that we can still do.
We first look into the straightforward comparison: AEG (A) and AEG (B). They only evaluate AEG (A) against a ResNet-18 \cite{resnet}. They do not provide benign accuracies for AEG (A) which is problematic. We therefore cannot compute a relative score to attempt a comparison. We still report the ESR they report in our overall Table \ref{tab:cifar10}, however, we italicize the result to indicate that it is unreliable.

Our other option would be to compare AEG (B) with Pixle, as AEG (B) provides benign accuracies, and they share a target model architecture (ResNet-18). Unfortunately, Pixle does not share its benign accuracy, so we deem their results unreliable for comparison.

We get the following results in Table \ref{table:ind_un_cifar10_red_comparison}.

\begin{table}[ht]
    \centering

    \begin{tabular}{@{}ccccccc@{}}
        \toprule
        \textbf{Model}  & \begin{tabular}[c]{@{}c@{}}\textbf{Benign}\\ \textbf{ESR}\end{tabular} & \begin{tabular}[c]{@{}c@{}}\textbf{AEG (B)}\\ \textbf{ESR / Score}\end{tabular} & \begin{tabular}[c]{@{}c@{}}\textbf{Pixle}\\ \textbf{ESR / Score}\end{tabular}\\
        \midrule
        ResNet-18 \cite{resnet} & 13.1 & 97.3 / 0.930 & \textit{\textbf{100.0}} / \textit{\textbf{0.983}}\\
        \bottomrule
    \end{tabular}
        \caption{CIFAR10 dataset indistinguishable AEG (B) \& Pixle undefended results.}
    \label{table:ind_un_cifar10_red_comparison}
\end{table}

\section{Takeaways}
\begin{enumerate}[noitemsep, topsep=0pt]

\item \textbf{Undefended models are broken.}
This is not a particularly new finding, but we confirm and reinforce it. We show that across all datasets and threat models studied, not once
are undefended models robust to attacks. On CIFAR-10, the undefended worst and best scores are, respectively, 0.219, and 0.994. Whereas, the defended worst and best scores are, respectively, 0.023, and 0.163. Even the best attack on defended models does not perform as well as the worst attack on undefended models.

\item \textbf{The transferable setting might not be as difficult as previously thought.}
We noticed that transferable attacks, when given additional information like access to the training data and training information such as the training function, can perform very well, almost on par with white-box attacks. For CIFAR-10, the best tranferable attack is only within 7.3\% ESR / 0.146 score of the best white-box attack.
This reinforces the notion that any information given to the attacker is crucial.

\item \textbf{Information categories matter, but their effect on performance is not straightforward.}
Information's effect on performance is highly dependent on the other information available. For example, in the white-box setting, data and training information are nominal. But in the transferable setting, they are crucial. In particular, training data information seems to be critical to attack performance. This is best observed when looking at BIA (A) on the CIFAR-10
dataset. They severely underperform other transferable attacks (47.2\% ESR / 0.22 score compared to 87.0\% ESR and 92.6\% ESR / 0.85 score for its closest competitors AEG(A) \& (B)). Not having access to the training data or even data from the same distribution appears to severily affect attack performance.

\item \textbf{Training information is understated but essential to transferable attacks.}
Most of the transferable attacks whose results we have studied use the training function. This is done to train surrogate models. However, this assumption is not explicitly stated in the papers we survey and is instead taken for granted. It goes against proper security practices as it omits  what our results clearly
show to be critical information in the threat model description.

\item \textbf{Need for better standard evaluation frameworks.}
There is a surprising lack of standards to evaluate adversarial example attacks. This unnecessarily complicates comparisons between existing works that are not directly 
related. Only two papers (ACG and A$^3$) used an existing evaluation framework (Robustbench \cite{robustbench}). 

\item \textbf{Evaluation against defended models.}
Too few papers evaluate their attack against defended models (and they usually evaluate too few defended models as well when they do). Only one paper evaluates its attack against defended models on ImageNet and only four papers do so on CIFAR-10. Papers that only evaluate undefended models bring little to the field as a whole as we have seen it is not a meaningful measure of attack quality.

\end{enumerate}

\section{Conclusion}
In this work, we presented a formalization to study adversary knowledge in the context of adversarial example attacks on image classification models. We were able to systematize a large body of existing work into our framework and are confident in its ability to extend to future research in adversarial example attacks and defenses.Our work can be used to enhance the reproducibility, comparability and fairness of such future work as it allows
a precise description of one's threat model and evaluations.

\section*{Acknowledgments}
We gratefully acknowledge the support of the Natural Sciences and Engineering Research Council (NSERC) for grants RGPIN-2023-03244, and IRC-537591 and the Royal Bank of Canada.

\bibliographystyle{unsrtnat}
\bibliography{ms}

\appendix
\section{Proofs}\label{app:proofs}
We can prove the ordering of this Hasse diagram using definition \ref{def:sqsubset}. We will first state the following lemmas and use them to prove Theorem \ref{theorem:model_hasse_diagram}. We will prove them afterwards.

\begin{lemma}\label{lemma:s_and_a_subset_m}
    $\mathcal{O}_{S\&A} \sqsubset \mathcal{O}_M$ and $\mathcal{O}_{M} \not\sqsubset \mathcal{O}_{S\&A}$
\end{lemma}

\begin{lemma}\label{lemma:l_subset_s}
    $\mathcal{O}_L \sqsubset \mathcal{O}_S$ and $\mathcal{O}_S \not\sqsubset \mathcal{O}_L$
\end{lemma}

\begin{lemma}\label{lemma:spa_subset_a}
    $\mathcal{O}_{SPA} \sqsubset \mathcal{O}_A$ and $\mathcal{O}_{A} \not\sqsubset \mathcal{O}_{SPA}$
\end{lemma}

\begin{proof}[Proof: Theorem: \ref{theorem:model_hasse_diagram}]
    First, $\mathcal{O}_{SPA} \sqsubset \mathcal{O}_{L\&SPA}$, $\mathcal{O}_{S} \sqsubset \mathcal{O}_{S\&SPA}$, $ \mathcal{O}_{S} \sqsubset \mathcal{O}_{S\&A}$, $ \mathcal{O}_{A} \sqsubset \mathcal{O}_{L\&A}$, and $\mathcal{O}_L \sqsubset \mathcal{O}_{L\&A}$ follow directly from definition \ref{def:ieo_combination} and part 2(a) of definition \ref{def:sqsubset}. \\
    Likewise, $\mathcal{O}_{L\&SPA} \not\sqsubset \mathcal{O}_{SPA}$, $\mathcal{O}_{S\&SPA} \not\sqsubset \mathcal{O}_{S}$, $ \mathcal{O}_{S\&A} \not\sqsubset \mathcal{O}_{S}$, $\mathcal{O}_{L\&A} \not\sqsubset \mathcal{O}_{A}$, and $\mathcal{O}_{L\&A} \not\sqsubset \mathcal{O}_{L}$ follow directly from definition \ref{def:ieo_combination} and part 2(b) of definition \ref{def:sqsubset}. \\
    Using Lemma \ref{lemma:l_subset_s} and part 2(c) of definition \ref{def:sqsubset}, we yield that $\mathcal{O}_{L\&A} \sqsubset \mathcal{O}_{S\&A}$, $\mathcal{O}_{S\&A} \not\sqsubset \mathcal{O}_{L\&A}$, $\mathcal{O}_{L\&SPA} \sqsubset \mathcal{O}_{S\&SPAA}$, and $\mathcal{O}_{S\&SPA} \not\sqsubset \mathcal{O}_{L\&SPA}$. \\
    Using Lemma \ref{lemma:spa_subset_a} and part 2(c) of definition \ref{def:sqsubset}, we get that $\mathcal{O}_{S\&SPA} \sqsubset \mathcal{O}_{S\&A}$ and $\mathcal{O}_{S\&A} \not\sqsubset \mathcal{O}_{S\&SPA}$. \\
    Using all of this in addition to Lemma \ref{lemma:s_and_a_subset_m}, we get that Theorem \ref{theorem:model_hasse_diagram} holds under the $\sqsubset$ ordering.
\end{proof}

We now move on to proving the various Lemmas we used to prove Theorem \ref{theorem:model_hasse_diagram}.

\begin{proof}[Proof: Lemma \ref{lemma:s_and_a_subset_m}]
To show $\mathcal{O}_{S\&A} \sqsubset \mathcal{O}_M$, we need to find PPT $f$ such that, $\forall a \in \{0, 1\}^*, \\\mathcal{O}_{S\&A}(a) = f(\mathcal{O}_M(a))$. Incidentally, we also need that $State(\mathcal{O}_{S\&A}) = f(\mathcal{O}_M(a))$ when both oracles are queried the same number of times.
We will have two cases, Case 1 where the attacker has unlimited queries, and Case 2 where the attacker is limited. \\
For the sake of simplicity, we assume that $a \in \mathcal{I}$ (as otherwise they both return $[]$, and we can set $f$ to return $[]$ when it receives $[]$ anyway and $\mathcal{O}_{S\&A}(a) = f(\mathcal{O}_M(a))$).
\paragraph{Case 1:}
Let $f$ be the following:
\begin{itemize}[noitemsep, topsep=0pt]
    \item On input $x \in \mathcal{I}$, $f$ receives $[\theta_M, x] = \mathcal{O}_M(x)$. From $\theta_M$, we can construct $M$ and $\phi_M$ (by definition of $\theta_M$).
    \item We can then compute and return $[M(x), \phi_M, x] = \mathcal{O}_{S\&A}(x)$.
\end{itemize}

\paragraph{Case 2:}
Let $f$ be the following and $k$ the current allowed number of queries left:
\begin{itemize}[noitemsep, topsep=0pt]
    \item On input $x \in \mathcal{I}$, $f$ receives $[\theta_M, x] = \mathcal{O}_M(x)$. From $\theta_M$, we can construct $M$ and $\phi_M$ (by definition of $\theta_M$).
    \item If $k = 0$, return $[[], \phi_M, x] = \mathcal{O}_{S\&A}(x)$.
    \item Else, update $State(\mathcal{O}_M(x))$ to be the function that returns $[k-1]$, return \\ $[M(x), \phi_M, x] = \mathcal{O}_{S\&A}(x)$.
\end{itemize}
So we have that $\forall a \in \{0,1\}^*$, $\mathcal{O}_{S\&A}(a) = f(\mathcal{O}_M(a))$. Hence, $\mathcal{O}_{S\&A} \sqsubset \mathcal{O}_M$.

Now for the opposite, $\mathcal{O}_{M} \not\sqsubset \mathcal{O}_{S\&A}$. For simplicity's sake, we'll cover only the first case, but the proof also applies to the second case by performing the update to the $State$ function as we did above.
We'll prove this by contradiction. Given $\mathcal{O}_{S\&A}$, an attacker can receive $\phi_M$ and $M(x)$ for any $x \in \mathcal{I}$. We are trying to reconstruct $\theta_M$. In the case where $M$ is linear, with enough queries, we can solve the 
system of linear equations and reconstruct $\theta_M$. However, in the case where $M$ is non-linear, there can be an infinite number of possible viable solutions for any given finite amount of queries, therefore it is impossible to always determine with certainty the 
exact $\theta_M$ (although it is possible to approximate it, as model stealing attacks demonstrate). Hence, it is not always possible to reconstruct $\theta_M$ from $\phi_M$ and score query-access, meaning that there is no PPT $f$ s.t. $\forall a \in \{0, 1\}^*, f(\mathcal{O}_{S\&A}(a)) = \mathcal{O}_M(a)$.
\end{proof}

\begin{proof}[Proof: Lemma \ref{lemma:l_subset_s}]
    We will prove Lemma \ref{lemma:l_subset_s} in the same way we proved Lemma \ref{lemma:s_and_a_subset_m}, i.e. $\exists f$ such that $\forall a \in \mathcal{I}$, $f(\mathcal{O}_S) = \mathcal{O}_L$.
    \paragraph{Case 1. Unlimited queries:}
    Let $f$ be the following:
    \begin{itemize}[noitemsep, topsep=0pt]
        \item On input $x \in \{0,1\}^*$, if $x \notin \mathcal{I}$, then $f$ receives $[]$ and return $[]$.
        \item Else, on input $x \in \mathcal{I}$, $f$ receives $[M(x), x]$ and returns $[argmax(M(x)), x] = \mathcal{O}_L$.
    \end{itemize}

    \paragraph{Case 2. Limited queries:}
    Let $f$ be the following and $k$ the current allowed number of queries left:
    \begin{itemize}[noitemsep, topsep=0pt]
        \item On input $x \in \{0,1\}^*$, if $x \notin \mathcal{I}$, then $f$ receives $[]$ and return $[]$.
        \item Else, on input $x \in \mathcal{I}$, $f$ receives $[M(x), x]$ if $k > 0$ else it receives $[[], x]$.
        \item If $k=0$, return $[[], x] = \mathcal{O}_L$.
        \item Else, return $[argmax(M(x)), x] = \mathcal{O}_L$.
    \end{itemize}

For the other direction, similarly to Lemma \ref{lemma:s_and_a_subset_m}, we will prove only the first case as the proof can trivially be extended to the second case. Given $\mathcal{O}_L$, we want to find a PPT function $f$ s.t. $\forall a \in \mathcal{I}$,
$f(\mathcal{O}_L(a)) = \mathcal{O}_S(a)$. We will show that such a function cannot exist. For $f(\mathcal{O}_L(a)) = \mathcal{O}_S(a)$ to happen, it means that $f$ needs to always be able to at least compute the score value of the label itself. It is given no information beyond the label itself for any input $x \in \mathcal{I}$.
Assume the following setup. One party holds two models $M_1$ and $M_2$ that make identical score predictions except for one specific input $x^*$ where $M_1(x^*) = u$ and $M_2(x^*) = v$ and $u \neq v$, we will also assume that while the score changes when inferring on $x^*$, the predicted label remains the same for both.
We argue that it would be impossible for the other party that receives only the predicted label to distinguish between both models. This means that there exists at least one instance where there is no PPT function $f$ s.t. $\forall a \in \mathcal{I}$,
$f(\mathcal{O}_L(a)) = \mathcal{O}_S(a)$. Therefore, we showed the other direction also holds. \\
Hence, Lemma \ref{lemma:l_subset_s} holds.
\end{proof}
In this case, we do not need to modify $State(\mathcal{O}_S)$ as we query $\mathcal{O}_S$ exactly once whenever we query $\mathcal{O}_L$.

\begin{proof}[Proof: Lemma \ref{lemma:spa_subset_a}]
    $\mathcal{O}_{SPA} \sqsubset \mathcal{O}_{A}$ follows from part 1(b) of definition \ref{def:sqsubset} and $\mathcal{O}_{SPA}$'s definition. For the other direction, since $\mathcal{O}_{SPA}(x) = [\{\phi_0, \phi_1, \dots, \phi_k\}, x]$ for some positive integer $k$ and where for a uniformly randomly sampled $i \in [0, 1, \dots, k]$, $\phi_i = \phi_M$,
    a PPT function $f$ would be able to always determine which of the $k$ elements in the unordered set is the correct architecture. Given that the function is only given the unordered set, it would be akin to being able to always correctly guess a random number, which is not possible.
    Therefore, $\mathcal{O}_{A} \not\sqsubset \mathcal{O}_{SPA}$ and Lemma \ref{lemma:spa_subset_a} holds.
\end{proof}

\begin{proof}[Proof: Theorem \ref{theorem:data_hasse_diagram}]
    Theorem \ref{theorem:data_hasse_diagram} follows directly from definitions \ref{def:sqsubset} and \ref{def:ieo_combination}.
\end{proof}

\begin{proof}[Proof: Theorem \ref{theorem:train_hasse_diagram}]
    By definition, we have the oracles defined by $\mathcal{O}_T$ that are properly defined and ordered under $\sqsubset$. It remains to show that $\mathcal{O}_{Train}$ dominates any one of them and none of them dominates $\mathcal{O}_{Train}$.
    By how any $T_i$ is defined, we have that $Train \in T_i$ for any integer $i$. Since $T_i$ is an unordered set, we can apply part 1(b) of definition \ref{def:sqsubset} to yield that $\forall i \in \mathbb{Z}$, $\mathcal{O}_{T_i} \sqsubset \mathcal{O}_{Train}$.
    We'll show that $\forall i \in \mathbb{Z}$, $\mathcal{O}_{Train} \not\sqsubset \mathcal{O}_{T_i}$ by contradiction. By part 1(c) of definition \ref{def:sqsubset}, for any integer $i$, there needs to be a PPT function $f$ such that $f([T_i, x]) = [Train, x]$ for any $x \in \{0,1\}$.
    Hence, we need some PPT function that can extract $Train$ from $T_i$ (as we know it is in $T_i$ by definition of $T_i$). However, $T_i$ is an unordered set and $f$ only has access to $T_i$ (and the query itself $x$ but in this case it should not provide additional information about the ordering of $T_i$).
    So, for $f$ to be able to distinguish $Train$ from any other of the other $Train'$ functions in $T_i$ would be equivalent to being able to perfectly guess truly random numbers, which is impossible. Therefore, we have a contradiction and so $\forall i \in \mathbb{Z}$, $\mathcal{O}_{Train} \not\sqsubset \mathcal{O}_{T_i}$.
\end{proof}

\begin{proof}[Proof: Theorem \ref{theorem:defense_hasse_diagram}]
    $\mathcal{O}_{PA} \sqsubset \mathcal{O}_{FA}$ follows directly from part 2(c) and 1(b) of definition \ref{def:sqsubset} as well as $\mathcal{O}_{FA}$'s definition since $\varrho \in  \{\varrho_1, \dots, \varrho_k\}$. $\mathcal{O}_{FA} \not\sqsubset \mathcal{O}_{PA}$ follows from the impossibility of perfectly guessing truly random numbers (since $i$ s.t. $\varrho_i = \varrho$ is uniformly randomly sampled) and part 1(c) of definition \ref{def:sqsubset}.
    $\mathcal{O}_{SPD} \sqsubset \mathcal{O}_{PA}$ also follows directly from part 2(c) and 1(b) of definition \ref{def:sqsubset} as well as $\mathcal{O}_{SPD}$'s definition since $\rho \in  \{\rho_1, \dots, \rho_k\}$. $\mathcal{O}_{PA} \not\sqsubset \mathcal{O}_{SPD}$ also follows from the impossibility of perfectly guessing truly random numbers and part 1(c) of definition \ref{def:sqsubset}.
    
\end{proof}

\section{Attack Summary Table}
\label{app:appendix_attack_summary}
We complete Table \ref{tab:attack_summary} with the information from the papers not included in the evaluation in Section \ref{sec:evaluation} with Table \ref{tab:appendix_attack_summary}.

\begin{table}[h!]
    \centering
    \resizebox{\columnwidth}{!}{
    \begin{tabular}{@{}cccccc@{}}
        \toprule
        \begin{tabular}[c]{@{}c@{}}\textbf{Attack}\\ (Targeted)\end{tabular} & \textbf{Model} & \textbf{Data} & \textbf{Train}& \textbf{Metric} \\
        \midrule
        \begin{tabular}[c]{@{}c@{}}ODI \\ \cite{Byun_2022_CVPR} (t) \end{tabular}&  \begin{tabular}[c]{@{}c@{}}Possible \\ Architectures\end{tabular} &  \begin{tabular}[c]{@{}c@{}}Training\\Data\end{tabular} & \begin{tabular}[c]{@{}c@{}} Training \\ Function \end{tabular} & l$_\infty$ \\
        \midrule
        \begin{tabular}[c]{@{}c@{}}GA \\ \cite{liu2022transferable} \end{tabular}& \begin{tabular}[c]{@{}c@{}}Possible \\ Architectures\end{tabular} & \begin{tabular}[c]{@{}c@{}}Training\\Data\end{tabular} & \begin{tabular}[c]{@{}c@{}} Training \\ Function \end{tabular} & l$_\infty$ \\
        \midrule
        \begin{tabular}[c]{@{}c@{}}AI-FGTM \\ \cite{Zou_Duan_Li_Zhang_Pan_Pan_2022} \end{tabular}& \begin{tabular}[c]{@{}c@{}}Possible \\ Architectures\end{tabular} & \begin{tabular}[c]{@{}c@{}}Training\\Data\end{tabular} & \begin{tabular}[c]{@{}c@{}} Training \\ Function \end{tabular} & l$_\infty$ \\
        \midrule
        \begin{tabular}[c]{@{}c@{}}F-Attack \\  \cite{Li_decisionbased2022} \end{tabular}& Labels & \begin{tabular}[c]{@{}c@{}}Same\\Distribution\end{tabular} & $\emptyset$ & None\\
        \midrule
        \begin{tabular}[c]{@{}c@{}}Admix (A) \\ \cite{Wang_2021_ICCV} (both)\end{tabular}& \begin{tabular}[c]{@{}c@{}}Possible \\ Architectures\end{tabular} &  \begin{tabular}[c]{@{}c@{}}Training\\Data \& Same \\Distribution\end{tabular} &  \begin{tabular}[c]{@{}c@{}} Training \\ Function \end{tabular} & None\\
        \addlinespace
        \begin{tabular}[c]{@{}c@{}}Admix (B) \\ \cite{Wang_2021_ICCV} (both)\end{tabular}& Architecture & \begin{tabular}[c]{@{}c@{}}Training\\Data \& Same \\Distribution\end{tabular} &  \begin{tabular}[c]{@{}c@{}} Training \\ Function \end{tabular} & None\\
        \addlinespace
        \begin{tabular}[c]{@{}c@{}}Admix (C) \\ \cite{Wang_2021_ICCV} (both)\end{tabular} & Parameters &  \begin{tabular}[c]{@{}c@{}}Same \\Distribution\end{tabular} &  \begin{tabular}[c]{@{}c@{}} Loss \\ Function \end{tabular} & None\\
        \midrule
        \begin{tabular}[c]{@{}c@{}}Shadow \\ \cite{Zhong_2022_CVPR} \end{tabular}& Scores & Other Data & $\emptyset$ & None \\
        \midrule
        \begin{tabular}[c]{@{}c@{}}DAPatch \\ \cite{chen2022shape} \end{tabular}& Parameters & $\emptyset$ & $\emptyset$ & l$_0$\\
        \midrule
        \begin{tabular}[c]{@{}c@{}}S$^2$I (A) \\ \cite{long2022frequency} (both)\end{tabular}& \begin{tabular}[c]{@{}c@{}}Possible \\ Architectures\end{tabular} & \begin{tabular}[c]{@{}c@{}}Training\\Data\end{tabular} & \begin{tabular}[c]{@{}c@{}} Training \\ Function \end{tabular} & l$_\infty$ \\
        \addlinespace
        \begin{tabular}[c]{@{}c@{}}S$^2$I (B) \\ \cite{long2022frequency} (both)\end{tabular}& Parameters & $\emptyset$ & \begin{tabular}[c]{@{}c@{}} Loss \\ Function \end{tabular} & l$_\infty$\\
        \midrule
        \begin{tabular}[c]{@{}c@{}}AI-GAN \\ \cite{aigan2021} (t)\end{tabular}& Parameters & \begin{tabular}[c]{@{}c@{}}Training\\Data\end{tabular} & $\emptyset$ & l$_\infty$ \\
        \midrule
        \begin{tabular}[c]{@{}c@{}}ACA (A) \\ \cite{chen2023contentbased} \end{tabular}& \begin{tabular}[c]{@{}c@{}}Possible \\ Architectures\end{tabular} & \begin{tabular}[c]{@{}c@{}}Training \\ Data \& \\ Other Data\end{tabular} & \begin{tabular}[c]{@{}c@{}} Training \\ Function \end{tabular} & None\\
        \addlinespace
        \begin{tabular}[c]{@{}c@{}}ACA (B) \\ \cite{chen2023contentbased} \end{tabular}& Parameters & Other Data & \begin{tabular}[c]{@{}c@{}} Loss \\ Function \end{tabular} & None \\
        \addlinespace
        \begin{tabular}[c]{@{}c@{}}DiffAttack (A)\\ \cite{chen2023diffusion} \end{tabular}& \begin{tabular}[c]{@{}c@{}}Possible \\ Architectures\end{tabular} & \begin{tabular}[c]{@{}c@{}}Training \\Data \& \\ Other Data\end{tabular} & \begin{tabular}[c]{@{}c@{}} Training \\ Function \end{tabular} & FID \cite{fid}\\
        \addlinespace
        \begin{tabular}[c]{@{}c@{}}DiffAttack (B) \\ \cite{chen2023diffusion} \end{tabular}& Parameters & Other Data & \begin{tabular}[c]{@{}c@{}} Loss \\ Function \end{tabular} & FID \cite{fid}\\
        \bottomrule

\end{tabular}
}
\caption{Attack Summary Table. Supplements Table \ref{tab:attack_summary}.}
\label{tab:appendix_attack_summary}
\end{table}

\end{document}